\documentclass[twoside,11pt]{article}

\usepackage{xcolor}
\usepackage{array}
\usepackage{amsmath}
\usepackage{jmlr2e}
\usepackage{amsfonts,braket}
\usepackage{graphicx}
\usepackage{caption}
\usepackage{subcaption}
\usepackage{algorithmic}
\usepackage{algorithm}

\newtheorem{assumption}{Assumption}
\newtheorem{preremark}{Example}
  {\begin{preremark}\upshape}{\end{preremark}}

\jmlrheading{1}{2019}{1-48}{6/19}{10/00}{meila00a}{Xun Zhang, William B. Haskell, and Zhisheng Ye}

\ShortHeadings{A Unifying Framework for Algorithms for Monotone inclusion problem  }{Xun Zhang, William B. Haskell, and Zhisheng Ye}
\firstpageno{1}

\begin{document}

\title{A Unifying Framework for Variance-Reduced Algorithms for Findings Zeroes of Monotone Operators}

\author{\name Xun Zhang  \email xunzhang@u.nus.edu \\
       \addr Department of Industrial and Systems Engineering\\ 
National University of Singapore\\ 
 21 Lower Kent Ridge Rd, Singapore
       \AND
       \name William B. Haskell \email wbhaskell@gmail.com \\
       \addr Krannert School of Management\\ Purdue University\\  West Lafayette, IN 47907
       \AND
       \name Zhisheng 	Ye \email yez@nus.edu.sg  \\
        \addr Department of Industrial and Systems Engineering\\ 
National University of Singapore\\ 
 21 Lower Kent Ridge Rd, Singapore
       }
\editor{}

\maketitle

\begin{abstract}
It is common to encounter large-scale monotone inclusion problems where the objective has a finite sum structure. 
We develop a general framework for variance-reduced  forward-backward splitting algorithms for this problem. 
This framework includes a number of existing deterministic and variance-reduced algorithms for function minimization as special cases,
and it is also applicable to more general problems such as saddle-point problems and variational inequalities. 
With a carefully constructed Lyapunov function, we show that the algorithms covered by our framework enjoy a linear convergence rate in expectation under mild assumptions.
We further consider Catalyst acceleration and asynchronous implementation to reduce the algorithmic complexity and computation time.
We apply our proposed framework to a policy evaluation problem and a  strongly monotone two player game, both of which fall outside of function minimization.
\end{abstract}

\begin{keywords}
  finite sum minimization, monotone inclusion, monotone operators, variance-reduced algorithms, Catalyst acceleration, asynchronous computing
\end{keywords}

\section{Introduction}

In the field  of convex optimization, many problems can be recast as \textit{monotone inclusion} problems, where the goal is to find zeros of  appropriate set-valued mappings called \textit{monotone operators}  \citep{monotoneiclusion3,survey}.
Some typical examples include  convex function minimization \citep{boyd},  convex-concave saddle-point problems \citep{survey}, and multi-agent decision making \citep{v-i}, among others.

In many cases, these monotone inclusion problems have a finite sum structure.
For example, finite sum minimization often arises in machine learning and statistics where we minimize the empirical risk (i.e., the average of individual loss functions corresponding to individual observations)  \citep{Dumais1998text,Cortes1995svm}.
In addition, using the ``factored'' technique \citep{Balamurugan2016}, the saddle-differential operators of bilinear saddle-point problems can also be split into a finite sum.
Furthermore, the game mappings of many multi-agent problems are expectations over a set of scenarios \citep{game_theory}.
Sample average approximation (SAA) can be applied to get a numerical solution and the resulting approximation  has a finite sum structure.

As a special case of the finite sum monotone inclusion problem, finite sum function minimization is ubiquitous in machine learning.
The most well known algorithm for finite sum minimization is full gradient descent (GD) \citep{Nesterov:2014:ILC:2670022}, which has a linear convergence rate but is computationally expensive for large-scale finite sum problems.
In view of this weakness, stochastic gradient descent (SGD) \citep{robbins1951} was proposed. 
SGD only computes the gradient of one function in each iteration, 
and thus requires much less computational power per-iteration compared to GD. 
Given the same number of gradient evaluations,
it has been shown both theoretically and empirically that SGD is able to achieve a lower empirical loss compared to GD, especially in early iterations  \citep{Liu1989,1980,Ber_cv}. 

Although SGD has many advantages over GD, an intrinsic disadvantage of SGD is that it only has a sublinear convergence rate in expectation \citep{robbins1951}. 
To improve upon the performance of SGD while retaining its advantages, a large class of variance-reduced algorithms has been proposed. 
This class of algorithms includes stochastic variance-reduced gradient (SVRG) \citep{JohnsonSVRG,Lin:SVRG}, SAGA \citep{DefazioSAGA}, hybrid stochastic average gradient (HSAG) \citep{reddi2015variance}, and stochastic average gradient (SAG) \citep{SAG:2012,Schmidt2017}, among others. 
The core idea of these algorithms is to replace the random gradient estimator in SGD with a dedicated variance-reduced estimator. 
All of these algorithms achieve a linear convergence rate in expectation and still inherit the low per-iteration cost of SGD. 
Many numerical experiments and real world applications have demonstrated that variance-reduced algorithms are able to achieve a lower empirical loss than both GD and SGD given the same computational budget \citep{DefazioSAGA,JohnsonSVRG,SAG:2012,reddi2015variance,Schmidt2017,Lin:SVRG}. 

In view of the  effectiveness of variance-reduced algorithms for doing finite sum function minimization, a natural idea is to use  similar algorithms to solve the more general finite sum monotone inclusion problem. In the original papers \citet{DefazioSAGA,JohnsonSVRG,reddi2015variance,Lin:SVRG}, 
the convergence analyses of these variance-reduced algorithms (for function minimization) are all done in a case by case manner. The main objective of this study is now to bridge this gap by providing a unifying convergence analysis of a class of variance-reduced algorithms for finite sum monotone inclusion problems. 
It is known that  GD and SGD  can be viewed as special cases of the classic and randomized Forward-Backward (FB) splitting method for solving monotone inclusion problems, respectively \citep{survey,FBSGD}.
In the existing literature, \citet{Balamurugan2016} has established the linear convergence of SVRG and SAGA for saddle-point problems. 
Our present work builds on \citet{Balamurugan2016} by unifying and generalizing their convergence analysis to more general variance-reduced algorithms for solving finite sum monotone inclusion problems. 

In our framework, we maintain a proxy for each of the individual operators in the forward step of the randomized FB splitting.
These proxies are then used to construct an unbiased estimator of the exact operator in the forward step.
By carefully designing the proxies and the estimator, the above framework is able to recover many existing algorithms for finite sum minimization
such as GD, SGD, SAGA, SVRG, and HSAG.
We show that after each iteration, the expected distance to the optimum shrinks geometrically subject to a disturbance 
which depends on the variance of the estimator.
By controlling the variance of the proxies, we are able to construct a Lyapunov function which consists of the distance to the optimum plus an error term determined by the value of proxies, and we show that this Lyapunov function decreases geometrically in expectation.
This result is then applied to GD, SAGA, HSAG, and SVRG with possibly varying epoch lengths.
Further, the framework allows us  to propose some new variance-reduced algorithms within our framework that have a linear convergence rate in expectation.
As an example, we propose a new algorithm called SVRG-rand which randomly does a full operator evaluation, instead of following a fixed epoch schedule. Numerical studies reveal that SVRG-rand performs better than both classical SVRG and SVRG with increasing epochs.

To further accelerate these variance-reduced algorithms, we extend Catalyst in  \cite{Catalyst}  for finite sum minimization problems to our variance-reduced FB splitting  monotone inclusion framework.
We show that in the ``ill-conditioned'' setting, many algorithms covered by our general randomized FB splitting framework attain a lower complexity after Catalyst acceleration. 
Apart from Catalyst acceleration,  we also consider parallel computing to reduce the computation time.
We find that our proposed framework can be extended to asynchronous versions of variance-reduced randomized forward-step algorithms. 
In the asynchronous case, in contrast to the synchronous case, 
we have a multicore structure for parallel computing where each of the cores may not have up-to-date information \citep{mania2015}. 
Under similar assumptions on the proxies as for synchronous algorithms, we show that the iterates generated by asynchronous variance-reduced algorithms converge linearly in expectation to a tolerance of the optimum.
Specifically, our results apply to asynchronous variants of SVRG, SAGA, HSAG, and SVRG-rand.
In contrast to previous work on asynchronous algorithms which required sparsity assumptions \citep{ASAGA,asgd,mania2015}, our results hold more generally.
In particular, the sparsity assumption is only applicable to finite sum minimization but does not hold for monotone inclusion problems (see \citet{Policyeva}).

This paper is organized as follows. 
Section \ref{sec:preli} introduces the notation and basic concepts of monotone operators and our problem setting.  
Section \ref{sec:forward-backward} presents a unifying	 framework for randomized
FB splitting methods that recovers many classical algorithms. Here we also provide the convergence analysis for our general framework.
Section \ref{sec:Catalyst} presents  Catalyst acceleration for monotone inclusion problems.
The asynchronous extension of our general algorithmic framework is investigated in Section \ref{sec:asyn}.
Then in Section \ref{sec:examples}, we cover the specifics of several existing and new algorithms.
Section \ref{sec:numerical} provides a detailed comparison of the numerical performance for some of the classical algorithms as well as our new algorithms. 
We conclude the paper with a discussion of our broader themes and possible future research topics.
All proofs are organized together in the Appendix.

\section{Preliminaries}\label{sec:preli}
We first define key concepts from the monotone operator literature. Let $\Vert \cdot\Vert $ denote the Euclidean norm on $\mathbb{R}^{d}$ throughout.

\begin{definition}
(i) An \emph{operator} $F : \mathbb{R}^d \rightarrow \mathbb{R}^d$ is a subset of $\mathbb{R}^{d}\times \mathbb{R}^{d}$, where $F(x) = \set{ y \in \mathbb{R}^d | (x,y)\in F }$.

(ii) An operator $F$ is \emph{monotone} if $(u-v)^\top(x-y)\geq 0$ for all $u \in F(x)$ and $v \in F(y)$.

(iii) An operator $F$ is $\mu$-\emph{strongly monotone} for $\mu>0$ if $(u-v)^\top (x-y)\geq \mu \|x-y  \|^2$ for all $u \in F(x)$ and $v \in F(y)$.

(iv) An operator is \emph{maximal monotone} if it is monotone and there is no monotone operator that properly contains it.

(v) An operator $F$ on $\mathbb{R}^{d}$ is $L$-\emph{Lipschitz} if $\| u-v  \|\leq L\| x-y \|$ for all $u \in F(x)$ and $v \in F(y)$.
\end{definition}
\noindent If $F$ is Lipschitz, then it is single-valued and we may write $F(x) = y$ \citep{survey}. When $L<1$, $ F$ is a \emph{contraction}; when $L=1$, $F$ is \emph{non-expansive}. 

Now we introduce our problem setting. 
Suppose we have two maximal monotone operators $A$ and $B$ on $\mathbb{R}^d$, where $B=\frac{1}{n}\sum_{i=1}^n B_{i}$ is the average of a collection of operators $\{B_i\}_{i=1}^n$ (indexed by $[n] \triangleq \{1,\ldots,\,n\}$).
The monotone inclusion problem is to find $x \in \mathbb{R}^d$ such that:
\begin{eqnarray}
0\in A\left(x\right)+\frac{1}{n}\sum_{i=1}^nB_{i}\left(x\right). \label{eq:inclusion}
\end{eqnarray}
We introduce the following assumptions on the operators $A$ and $\{B_i\}_{i=1}^n$.

\begin{assumption}\label{assum1}
(i) $A$ is maximal monotone on $\mathbb{R}^{d}$.

(ii) $B=\frac{1}{n}\sum_{i=1}^n B_{i}$ is $\mu$-strongly maximal monotone on $\mathbb{R}^{d}$, 
and each $B_{i}$ is $L$-Lipschitz. 
\end{assumption}
Our assumptions are slightly different from those in \citet{Balamurugan2016}. 
We assume strong monotonicity for $B$ and Lipschitz continuity for all $B_i$ to align our work with the literature on finite sum minimization.

By Assumption \ref{assum1}, $A+B$ is strongly monotone and so (\ref{eq:inclusion}) has a unique solution $x^* \in \mathbb{R}^d$, and any $x \in \mathbb{R}^d$ such that $\| x-x^* \|^2\leq \epsilon$ is an $\epsilon$-solution of Problem \eqref{eq:inclusion}. The condition number of Problem \eqref{eq:inclusion} is $\kappa \triangleq L/\mu$, and it will appear frequently in our complexity bounds.

Problem (\ref{eq:inclusion}) is general and accommodates many problems in optimization as seen in the following examples.

\begin{example}[Finite sum minimization]
Let $f_i:\mathbb{R}^d\to\mathbb{R}$ for $i \in [n]$ be convex, $h:\mathbb{R}^d\to\mathbb{R}$ be convex, and $C \subset \mathbb{R}^d$ be a convex set. The finite sum minimization problem is:
\begin{equation}\label{eq:finite sum}
\min_{x\in C}~ f(x)+ h(x) \textrm{ where } f(x) \triangleq\frac{1}{n}\sum_{i=1}^{n}f_{i}\left(x\right).
\end{equation}
The function $f$ in (\ref{eq:finite sum}) is usually assumed to be both strongly convex (corresponding to strong monotonicity) and strongly smooth (corresponding to Lipschitz continuity of its gradient). Let $\partial$ be the subdifferential, and $I_C$ be the indicator function of the set $C$. Then, Problem \eqref{eq:finite sum} is equivalent to (\ref{eq:inclusion}) where $A = \partial h(x) + \partial I_C(x)$ and $B_i = \partial f_{i}(x)$ for all $i \in [n]$.
\end{example}

\begin{example}[Finite sum saddle-point problems]
Let $f_i:\mathbb{R}^{d_1} \times \mathbb{R}^{d_2}\to\mathbb{R}$ for $i \in [n]$ be convex/concave, $h:\mathbb{R}^{d_1} \times \mathbb{R}^{d_2}\to\mathbb{R}$ be convex/concave, and $C_1 \subset \mathbb{R}^{d_1}$ and $C_2 \subset \mathbb{R}^{d_2}$ be convex sets. The finite sum saddle-point problem is:
\begin{equation}\label{eq:saddle finite sum}
\min_{x\in C_1} \max_{y \in C_2} f(x,\,y)+ h(x,\,y) \textrm{ where } f(x,\,y) \triangleq\frac{1}{n}\sum_{i=1}^{n}f_{i}\left(x,\,y\right).
\end{equation}
Problem \eqref{eq:saddle finite sum} is equivalent to (\ref{eq:inclusion}) where
$A = (\partial_x h(x,\,y), -\partial_y h(x,\,y)) + \partial I_{C_1}(x) - \partial I_{C_2}(y)$
and $B_i = (\partial_x f_{i}(x,\,y), -\partial_y f_{i}(x,\,y))$ for all $i \in [n]$.
\end{example}

\begin{example}[Finite sum variational inequalities]
Let $B_i:\mathbb{R}^{d} \to \mathbb{R}^d$ for $i \in [n]$ be monotone, and let $C \subset \mathbb{R}^d$ be a convex set. The finite sum variational inequality is to find $x^* \in \mathbb{R}^d$ such that:
\begin{equation}\label{eq:variational finite sum}
\langle B(x^*),\,y-x^*\rangle \geq 0,\, \forall y \in C.
\end{equation}
Let $N_C$ be the normal cone to the set $C$. Then, Problem \eqref{eq:variational finite sum} is equivalent to (\ref{eq:inclusion}) where $A = N_C$ and $B_i$ is as given for all $i \in [n]$.
\end{example}

The most classical algorithm for solving (\ref{eq:inclusion}) is forward-backward (FB) splitting \citep{monotoneiclusion3,FBspit} given by:
\begin{eqnarray}
x_{k+1} & = & \left(I+\gamma\,A\right)^{-1}\left(I-\gamma\,B\right)\left(x_{k}\right),\,\forall k \geq 0,\label{eq:deterministic}
\end{eqnarray}
where $\gamma > 0$ is the step-size. The following theorem demonstrates the linear convergence of the sequence produced by (\ref{eq:deterministic}) in terms of the distance to the solution of (\ref{eq:inclusion}).

\begin{theorem}\label{FB-thm}
Suppose \hyperref[assum1]{Assumption~\ref*{assum1}} holds and
let $\left\{ x_{k}\right\} _{k\geq0}$ be the sequence produced by (\ref{eq:deterministic}). 
Then, for all $k\geq 0$ we have
$$
\|x_{k+1}-x^{*}\|^{2}\leq\left(1-2\,\gamma\,\mu+\,\gamma^{2}L^{2}\right)\|x_{k}-x^{*}\|^{2}.
$$ 
When $\gamma \in \left(0, 2\,\mu/L \right)$, $\{x_k\}_{k \geq 0}$ converges linearly to the unique solution $x^*$ of (\ref{eq:inclusion}).
\end{theorem}

\section{Variance-reduced forward-backward splitting}\label{sec:forward-backward}

In this section, we present a framework for variance-reduced FB splitting for solving \eqref{eq:inclusion}. In classical FB splitting, we need to compute $B(x_k)$ exactly in every iteration $k \geq 0$, but this step requires computation of $B_i(x_k)$ for every $i \in [n]$. Our framework is based on introducing a proxy variable for every operator $\{B_i\}_{i=1}^n$ which records the most recent evaluation. We use the proxies to construct an estimator for $B(x_k)$ and then strategically update the proxies to avoid evaluating all $\{B_i\}_{i=1}^n$ in every iteration (this strategy is generally true of all variance-reduced first-order methods for optimization).

We denote the proxies by $\phi^{k}=(\phi_{i}^k)_{i \in [n]}$ for all $k \geq 0$, where $\phi_i^k$ is the proxy for the most recent evaluation of $B_i$ in iteration $k \geq 0$. We also introduce a sequence of i.i.d. discrete uniform random variables $\left\{ I_{k}\right\} _{k\geq0}$ with support on $[n]$. Then, we define the estimators:
\begin{equation}\label{eq:estimator}
\mathcal{G}(x_{k},\,\phi^k,I_{k}) \triangleq B_{I_{k}}\left(x_k \right)-\phi_{I_k}^k+\frac{1}{n}\sum_{i=1}^n \phi_{i}^k,\,\forall k \geq 0,
\end{equation}
which we use in place of $B(x_k)$. The variance-reduced FB splitting method is then
\begin{equation}\label{eq:randomFB}
x_{k+1}=\,  \left(I+\gamma_k\,A\right)^{-1}\left(I-\gamma_k\,\mathcal{G}(x_k,\,\phi^k,\,I_k)\right),\,\forall k \geq 0,
\end{equation}
where $\{\gamma_k\}_{k \geq 0}$ is a sequence of positive step-sizes. The proxies are updated according to some scheme which we denote by:
\begin{align}\label{eq:update_scheme}
\phi^{k+1}= \mathcal{U}_{k}\left(x_{k},\,\phi^{k},\,I_{k}\right),\,\forall k \geq 0.
\end{align}
The particular rule for updating the proxies $\phi^k$ depends on the specific algorithm. This update rule is allowed to depend on the iteration number to allow for epoch-based algorithms like SVRG. The complete description of variance-reduced FB splitting is given in \hyperref[alg:Framwork]{Algorithm~\ref*{alg:Framwork}}. The algorithm is essentially determined by the update rule $\{\mathcal{U}_k\}_{k \geq 0}$ in (\ref{eq:update_scheme}) since the gradient estimator \eqref{eq:estimator} is almost always the same.

Let $\left\{ \mathcal{F}_{k}\right\} _{k\geq0}$ be a filtration where $\mathcal{F}_{k}$ denotes the history of the algorithm up to the $k$th iteration.
Because $I_k$ is uniform, it is readily seen that $\mathcal{G}\left(x_{k},\,\phi^k,I_{k}\right)$ is an unbiased estimator of $B(x_k)$ for any proxy updating scheme.

\begin{lemma}\label{lemma:unbiased}
For all $k \geq 0$, $\mathbb{E}\left[\mathcal{G}\left(x_{k},\,\phi^k,I_{k}\right) | \mathcal{F}_{k}\right]=B\left(x_{k}\right)$.
\end{lemma}

\begin{algorithm}[htb] 
\caption{Variance-reduced forward-backward splitting}
\label{alg:Framwork} 
\begin{algorithmic}[1] 
\REQUIRE ~~\\ 
Initial values: $x_0\in \mathbb{R}^d$ and $\phi_i^0\in \mathbb{R}^d$ for all $i\in [n],$ total number of iterations $T$.
\FOR{each $k =0,1,\ldots,T$}
\STATE Generate $I_k$ uniformly in $[n]$.
\STATE $x_{k+1}=\,  \left(I+\gamma_k\,A\right)^{-1}\left(x_{k}-\gamma_k\, \mathcal{G}\left(x_{k},\,\phi^k,I_{k}\right) \right)$.
\STATE $\phi^{k+1}=\,  \mathcal{U}_{k}(x_{k},\,\phi^{k},\,I_{k}).$
\ENDFOR

\ENSURE ~~\\ 
\STATE $x_{T+1}$
\end{algorithmic}
\end{algorithm}

We start our analysis by showing how the expected distance to the optimal solution $x^*$ changes after one iteration.
This upcoming inequality determines the convergence rate of Algorithm \ref{alg:Framwork} and shows how it depends on the variance of the estimator $\mathcal{G}\left(x_{k},\,\phi^k,I_{k}\right)$.

\begin{lemma}\label{lemma1}
Suppose \hyperref[assum1]{Assumption~\ref*{assum1}} holds and let $\left\{ x_{k}\right\} _{k\geq0}$ be produced by Algorithm \ref{alg:Framwork}.
Then, for all $k\geq0$ we have
\begin{align}
\mathbb{E} \|x_{k+1}-x^{*}\|^{2} \leq\left(1-\,2\gamma_k\,\mu+\,\gamma_k^{2}L^{2}\right)\mathbb{E} \|x_{k}-x^{*}\|^{2} +\,\gamma_k^{2}\mathbb{E} \|\mathcal{G}\left(x_{k},\,\phi^k,I_{k}\right)-B\left(x_{k}\right)\|^{2} .\label{eq:stochastic-1}
\end{align}
\end{lemma}

In the next lemma, we bound the conditional variance of $\mathcal{G}(x_k,\phi^k,I_k)$.
\begin{lemma}\label{lemma2}
Suppose \hyperref[assum1]{Assumption~\ref*{assum1}} holds and let $\{x_k \}_{k\geq 0}$ be generated by Algorithm \ref{alg:Framwork}. Then, for all $k\geq0$ we have
$$
\mathbb{E}\bigg\|\mathcal{G}\left(x_{k},\,\phi^k,I_{k}\right)-B\left(x_{k}\right)\bigg\|^{2}\leq2\left(L^2\,\mathbb{E}\|x_{k}-x^{*}\|^{2}+\mathbb{E}\frac{1}{n}\sum_{i=1}^{n}\|\phi_{i}^{k}-B_{i}\left(x^{*}\right)\|^{2}\right).
$$
\end{lemma}

The following corollary is immediate from the bounds in Lemma \ref{lemma1} and Lemma~\ref{lemma2}.

\begin{corollary}\label{coro1}
Suppose \hyperref[assum1]{Assumption~\ref*{assum1}} holds and let $\{x_k \}_{k\geq 0}$ be generated by Algorithm \ref{alg:Framwork}. 
Then, for all $k\geq 0$ we have
\begin{align*}
\mathbb{E}\|x_{k+1}-x^{*}\|^{2} \leq (1-2\gamma_k \mu+3\gamma_k^{2}L^2)\mathbb{E} \|x_{k}-x^{*}\|^{2} 
+\frac{2\gamma_k^2}{n}\mathbb{E}\sum_{i=1}^n\|\phi_{i}^k-B_{i}(x^*) \|^2.
\end{align*}
\end{corollary}

Suppose $1-2\,\gamma_k \mu+3\,\gamma_k^{2}L^2<1$ holds for all $k \geq 0$. Then, in each iteration, the expected distance to $x^*$ is upper bounded by a factor less than one of the current distance to $x^*$, plus an additional disturbance. 
The additional disturbance depends on how close the proxies are to the associated operator values $\{B_i(x^*)\}_{i = 1}^n$ at $x^*$.
Different algorithms give rise to different disturbances (because they update the proxies differently), and thus yield different convergence rates.

To allow for epoch-based algorithms, we introduce the sequence $\{m_i\}_{i \geq 1}$ of epoch lengths where $m_i$ is the length of epoch $i$ and all $m_i \geq 1$. We also introduce the sequence $\{S_i\}_{i \geq 0}$ where $S_0=0$ and $S_i=\sum_{j=1}^i m_j$ for all $i\geq 1$, which denote the starting indices of each epoch. We then let 
\begin{align*}
\tilde{x}_k \triangleq x_{S_{k}},~~~\tilde{\phi}^k \triangleq \phi^{S_k},\,\forall k \geq 0,
\end{align*}
denote the iterates obtained at the end of each epoch. For non-epoch based algorithms, we have $m_i = 1$ and $S_i = i$ for all $i \geq 1$. 

In epoch-based algorithms, some proxies are only updated at the end of an epoch. To model this scheme, we divide the index set $[n]$ into a set $\mathcal{S}\subseteq[n]$ and its complement $\mathcal{S}^c$. The proxies in $\mathcal{S}$ follow an algorithm-specific update, while the proxies in $\mathcal{S}^c$ are only updated at the end of each epoch.

Now we construct a Lyapunov function to analyze Algorithm \ref{alg:Framwork}. First, we define
\begin{eqnarray*}
G(\phi^k)&\triangleq& \frac{1}{n}\sum_{i\in \mathcal{S}}\|\phi_{i}^k-B_{i}(x^*) \|^2,\,\forall k \geq 0,\\
H(\phi^k)&\triangleq& \frac{1}{n}\sum_{i \notin \mathcal{S}}\|\phi_{i}^k-B_{i}(x^*) \|^2,\,\forall k \geq 0,
\end{eqnarray*}
where we stipulate that $G(\phi^k)\equiv 0$ if $\mathcal{S}=\emptyset$ and $H(\phi^k)\equiv 0$ if $\mathcal{S}^c=\emptyset$. It is readily seen that
$G(\phi^k)+H(\phi^k) = \frac{1}{n}\sum_{i=1}^n\|\phi_{i}^k-B_{i}(x^*) \|^2$, so we have split all the proxies into two groups which follow different update schemes.
For example, when considering HSAG \citep{reddi2015variance}, $G$ and $H$ will capture the proxies that follow SAGA-type updates and SVRG-type updates, respectively.  
Now for $\rho \geq 0$, we define the Lyapunov function
$$
L_{\rho}\left(x_{k},\phi^k \right)\triangleq\|x_{k}-x^{*}\|^2 + \rho\, G(\phi^k).
$$
Our analysis will focus on $L_{\rho}(\tilde{x}_{k},\tilde{\phi}^{k})$, the value of the Lyapunov function at the end of each epoch.

Next we introduce our assumptions on $G$ and $H$. Under these assumptions, we can show that the sequence $\{L_{\rho}(\tilde{x}_{k},\tilde{\phi}^{k})\}_{k \geq 0}$ converges geometrically to zero in expectation which gives the linear convergence rate of Algorithm \ref{alg:Framwork}.
\begin{assumption}\label{assum2}
~~~
\begin{itemize}
\item[(\ref{assum2}.1)] There exist constants $0\leq c_1<1$ and $c_2\geq0$ 
such that $\mathbb{E} G(\phi^{k+1}) \leq c_1 \mathbb{E} G(\phi^k) +c_2 \mathbb{E} \|x_{k}-x^{*}\|^{2}.$ 
\item[(\ref{assum2}.2)] There exists a constant $c_3 \geq 0$ such that $H(\phi^{k}) \leq  c_3  L_{\rho}(x_{S_i},\phi^{S_i})$ for all $ S_i\leq k < S_{i+1}$. Furthermore, $\bar{m}\triangleq \sup_{i \geq 1}m_i <\infty.$
\item[(\ref{assum2}.3)] The step-size is constant, i.e., $\gamma_k = \gamma$ for all $k\geq 0$.
\end{itemize}
\end{assumption}
\noindent
We will show that these assumptions are sufficient for Algorithm \ref{alg:Framwork} to have a linear convergence rate. Assumption \ref{assum2}.1 corresponds to the proxies in $\mathcal{S}$ that are updated in every iteration and Assumption \ref{assum2}.2 corresponds to the proxies in $\mathcal{S}^c$ that are updated only at the end of each epoch. When $\mathcal{S}=\emptyset$, Assumption \ref{assum2}.1 holds automatically since $G(\phi^k)\equiv 0$. 
On the other hand, when $\mathcal{S}=[n]$ then Assumption \ref{assum2}.2 holds automatically since $H(\phi^k)\equiv 0$.

Assumption \ref{assum2}.1 requires $G(\phi^k)$ (the sum of the errors for proxies in $\mathcal{S}$) to contract by a factor less than one, plus an additional disturbance which depends on $\|x_k - x^{*}\|$. This is essentially a condition on the proxy update scheme \eqref{eq:update_scheme} to ensure that the proxies are updated ``frequently enough''.

Assumption \ref{assum2}.2 requires $H(\phi^k)$ to be uniformly bounded over each epoch, where the upper bound depends on the value of the Lyapunov function at the beginning of the epoch (multiplied by a constant).  When proxies captured by $H(\phi^k)$ are updated only at $k=S_k$, then $H(\phi^k)$ is unchanged for all iterations $ S_i\leq k < S_{i+1}$. Furthermore, $H(\phi^k)$ usually depends on the distance to $x^{*}$ which can be further bounded by the value of the Lyapunov function.

Next we present our main result on the linear convergence of Algorithm \ref{alg:Framwork}. We define the following two constants:
\begin{align}
\theta&\triangleq \max\left\lbrace1-2\gamma\mu+3\gamma^2L^2+c_2\rho,~\frac{2\gamma^2}{\rho}+c_1  \right\rbrace \label{theta},\\
\lambda &\triangleq \theta+2\gamma^2\bar{m}c_3\label{-lambda}.
\end{align}
We also introduce the following step-size rule:
\begin{equation}\label{eq:gamma ineq}
\gamma<\min\left\lbrace\left(\frac{2\mu}{\frac{3(1-c_1)L^2}{2}+(1-c_1)c_3\bar{m}+c_2}\right)^2  ,~\left(\frac{1-c_1}{2+2\bar{m}\left(\frac{1-c_1}{2}\right)^3c_3}\right)^2\right\rbrace.
\end{equation}

\begin{theorem}\label{main theorem}
Suppose Assumptions~\ref{assum1} and \ref{assum2} hold, the step-size $\gamma$ satisfies \eqref{eq:gamma ineq}, and let $\{x_k\}_{k \geq 0}$ be generated by Algorithm \ref{alg:Framwork}. Set $\rho=\gamma^{1.5}$, then $\lambda\in[0,1)$ and
\begin{equation}\label{eq:L_rho contraction}
\mathbb{E}L_{\rho}\left(\tilde{x}_{k},\tilde{\phi}^{k}\right)\leq \lambda \mathbb{E} L_{\rho}\left(\tilde{x}_{k-1},\tilde{\phi}^{k-1}\right),\,\forall k \geq 1.
\end{equation} 
Furthermore,
\begin{equation}\label{eq:x_k contraction}
\mathbb{E}\|\tilde{x}_{k}-x^*  \|^2\leq \lambda^k(1+\rho L^2)\| x_0-x^* \|^2,\,\forall k \geq 1.
\end{equation}
\end{theorem} 

\hyperref[main theorem]{Theorem~\ref*{main theorem}} is based on a generic Lyapunov function that can be used for many variance-reduced algorithms. 
In contrast, the Lyapunov functions and the corresponding convergence analysis that appear in \citep{DefazioSAGA,JohnsonSVRG,reddi2015variance,Lin:SVRG} are all designed for the specific algorithm at hand.
We note that $\rho=\gamma^{1.5}$ is not the only allowable choice for the parameter of the Lyapunov function. Using very similar arguments as Theorem \ref{main theorem}, we can extend the choice $\rho=\gamma^{1.5}$ to $\rho=\gamma^t$ for $1<t<2$ and still retain linear convergence. However, the expression for the resulting step-size rule is more complicated.

\section{Acceleration by Catalyst}\label{sec:Catalyst}

Catalyst proposed by \cite{Catalyst} is a well-known method for acceleration of first-order algorithms for finite sum minimization problems.
To accelerate a generic first-order algorithm $\mathcal{M}$ for Problem \eqref{eq:finite sum}, 
\cite{Catalyst} uses a method based on an inner-outer loop structure. 
In the inner loop, the algorithm $\mathcal{M}$ is called to solve an auxiliary problem where a quadratic regularizer is added to the original objective function to increase its strong convexity.
The inner loop terminates when the auxiliary problem is solved by $\mathcal{M}$ up to a pre-determined error tolerance.
Then, in the outer loop, the iterate generated by the inner loop is used to define a new regularizer for the auxiliary problem of the next inner loop. 
For ``ill-conditioned'' problems where the condition number $\kappa = L/\mu$ is large compared to $n$, Catalyst will accelerate a large class of finite sum minimization algorithms.

In this section, we show that the regularization technique of Catalyst can also accelerate Algorithm \ref{alg:Framwork}. 
To this end, we adopt the inner-outer loop structure from \cite{Catalyst}.
In the inner loop,  we call an algorithm $\mathcal{M}$ to solve an auxiliary problem where an affine operator that mimics the role of the regularizer is added to the RHS of \eqref{eq:inclusion}.
Then, in the outer loop, we use the iterate generated by the inner loop to update the auxiliary problem.
In the following, we will compare the overall complexity of variance-reduced FB splitting with and without Catalyst acceleration.

Let $\mathcal{M}(A,B)$ denote a generic first-order algorithm for solving \eqref{eq:inclusion}. 
We assume $\mathcal{M}(A,B)$ has a linear convergence rate, 
i.e., there exists a constant $\pi_{\mathcal{M}(A,B)} > 0$ such that $\mathcal{M}(A,B)$  requires at most $O(\pi_{\mathcal{M}(A,B)} \log(1/\epsilon))$ operator evaluations to find an $\epsilon-$solution satisfying $\mathbb{E}\|x^*_\epsilon-x^*  \|^2\leq \epsilon$. 
The constant $\pi_{\mathcal{M}(A,B)} $ usually depends on the condition number $\kappa=L/\mu$. 
When extended to the monotone operator setting, it can be readily shown that popular algorithms like SVRG, SAGA, and HSAG, all have complexity $O\left(n\, \log\left(1/\epsilon\right)\right)$ when the problem is well-conditioned and $\kappa^2\leq n$ (see Table~\ref{table: 1}). This complexity is already optimal \citep{Catalyst,lowerbound_of_mini}.
We focus on the ``ill-conditioned'' case where $\kappa^2 \geq n$ and the complexity is dominated by $\kappa^2$.

\begin{algorithm}[htb] 
\caption{Catalyst}
\label{alg:Catalyst} 
\begin{algorithmic}[1] 
\REQUIRE ~~\\ 
An initial value $\check{x}_0 \in \mathbb{R}^d$,  the total number of inner-outer loops $T$, and parameter $\sigma$.
\FOR{each $k =0,1,\ldots,T$}
\STATE set $ \bar{x}=\check{x}_k$ and consider Problem \eqref{eq: auxiliary}.  Set the index for inner-loop $s=0$ and initialize the inner-loop with $ x_0=\bar{x}$.
\WHILE{ the desired stopping criterion is not satisfied}
\STATE Starting with ${x}_0$, use $\mathcal{M}(A,\, B + \sigma(I - \bar x))$ to generate the sequence $\{ {x}_s \}_{s\geq 0}$.
\STATE If Equation \eqref{eq: stopping criterion} holds for some ${x}_s$, then stop.
\ENDWHILE
\STATE $\check{x}_{k+1}={x}_{s}.$
\ENDFOR
\ENSURE ~~\\ 
\STATE $\check{x}_{T+1}$
\end{algorithmic}
\end{algorithm}

The details of Catalyst for solving \eqref{eq:inclusion} are shown in Algorithm \ref{alg:Catalyst}. 
We let $\{\check{x}_k\}_{k\geq 0}$ be the sequence of iterates at the beginning of each outer loop in Catalyst. 
In the $k$th outer loop we initialize with $\bar{x}=\check{x}_k$.  Then, in the inner loop we call 
$\mathcal{M}\big(A,\, B + \sigma(I - \bar x)\big)$ to solve the auxiliary monotone inclusion problem:
\begin{align}\label{eq: auxiliary}
0\in A(x)+B(x)+\sigma (x-{\bar{x}}),
\end{align}
where $\sigma \geq 0$ is a regularization parameter to be specified later. 
In \eqref{eq: auxiliary}, the operator $B+\sigma ({I}-\bar{x})$ is $(L+\sigma)$-Lipschitz and $(\mu+\sigma)$-strongly monotone, 
while the original operator $B$ is $L$-Lipschitz and $\mu$-strongly monotone. 
 Catalyst runs $\mathcal{M}(A,\, B + \sigma(I - \bar x))$ to solve \eqref{eq: auxiliary}  starting with ${x}_0 = \bar{x}$ and generates a sequence $\{{x}_s \}_{s\geq 0}$, until we find an iterate ${x}_s$ such that
\begin{align}\label{eq: stopping criterion}
\mathbb{E} \|{x}_{s}-{x}^*(\bar{x})\|^2\leq \frac{\mathbb{E}\|{x}_0-{x}^*(\bar{x})\|^2}{4(1+\sigma/\mu)^2},
\end{align}
where ${x}^*(\bar{x})$ is the unique solution of Problem \eqref{eq: auxiliary}.
Then, Catalyst takes $\check{x}_{k+1} = {x}_{s}$ and starts the next outer loop.

The following lemma establishes the linear convergence of Catalyst across outer loops.
\begin{lemma}\label{lem: linear of cata}
Suppose Assumption \ref{assum1} holds and let $\{\check{x}_k\}_{k \geq 1}$ be generated by Algorithm \ref{alg:Catalyst}. Then, for all $k \geq 0$ we have
\begin{align*}
\mathbb{E}\|\check{x}_{k+1}-x^*  \|^2\leq \left( 1-\frac{1}{2(1+\sigma/\mu)}  \right) \mathbb{E} \|\check{x}_k-x^* \|^2.
\end{align*}
\end{lemma}
We know that $(1-1/x)^x\thickapprox 1/e$ for $x \gg 0$.
The parameter $\sigma$ depends on $\kappa$ and is often large for ill-conditioned problems (see Section \ref{sec:examples}).
It follows that Catalyst needs $O(\frac{2(\mu+\sigma)}{\mu}\log(1/\epsilon))$ outer loops to compute an $\check{x}_k$ such that $\mathbb{E}\|\check{x}_k-x^*  \|^2\leq \epsilon$.
For each outer loop, we call $\mathcal{M}(A,\, B + \sigma(I - \bar x))$ to obtain a solution of accuracy $\frac{1}{4(1+\sigma/\mu)^2}$ for Problem \eqref{eq: auxiliary}, and the required number of operator evaluations is $O(\pi_{\mathcal{M}(A,B+\sigma(I-\bar{x}))}\log( \frac{2(\mu+\sigma)}{\mu}))$. 
In each inner loop, the parameter $\bar{x}$ is fixed and does not affect the order of $\pi_{\mathcal{M}(A,B+\sigma(I-\bar{x}))}$ and so it does not affect the overall complexity.
We then arrive at the following theorem on the total complexity of Catalyst for \eqref{eq:inclusion}.

\begin{theorem}
Suppose Assumption \ref{assum1} holds and let $\{\check{x}_k \}_{k=1}^T$ be generated by Algorithm \ref{alg:Catalyst}.
Then, the complexity of Catalyst to achieve $\mathbb{E}\|\check{x}_T-x^* \|^2\leq \epsilon$ is 
\begin{align}\label{prop:compl of cata}
O\left(\frac{\sigma+\mu}{\mu} \pi_{\mathcal{M}(A,B+\sigma I)}\log\left( \frac{2(\mu+\sigma)}{\mu}\right)\log\left(\frac{1}{\epsilon}\right)\right).
\end{align}
\end{theorem}

\noindent Note that when $\sigma=0$, the complexity of Catalyst is the same as that of the original $\mathcal{M}$. The optimal value $\sigma^*$ for $\sigma$ to achieve the optimal complexity is algorithm-dependent. 
The following table compares the complexity of various algorithms before and after Catalyst acceleration. 
\begin{table}[!htbp]
\renewcommand\arraystretch{1.1}
\begin{tabular}{|c|c|c|c|}
\hline 
Algorithms & without Catalyst & with Catalyst &  optimal $\sigma$ \\
\hline  
SAGA & $O(\kappa^2\log(1/\epsilon) )$ & $\tilde{O}(\kappa\sqrt{n}\log(1/\epsilon) )$  &$ O(\kappa\mu/\sqrt{n})$\\
\hline 
SVRG & $O(\kappa^2\log(1/\epsilon) )$ & $\tilde{O}(\kappa\sqrt{n}\log(1/\epsilon) )$ &$ O(\kappa\mu/\sqrt{n})$\\
\hline 
HSAG & $O(\kappa^2\log(1/\epsilon) )$ & $\tilde{O}(\kappa\sqrt{n}\log(1/\epsilon) )$ &$ O(\kappa\mu/\sqrt{n})$\\
\hline 
SVRG-rand & $O(\kappa^2\log(1/\epsilon) )$ & $\tilde{O}(\kappa\sqrt{n}\log(1/ \epsilon)$& $O(\kappa\mu/\sqrt{n})$ \\
\hline 
SAGA+SVRG-rand & $O(\kappa^2\log(1/\epsilon) )$ & $\tilde{O}(\kappa\sqrt{n}\log(1/\epsilon)$ & $O(\kappa\mu/\sqrt{n})$ \\
\hline 
SAGD & $O(\kappa^2\log(1/\epsilon) )$ & $\tilde{O}(\kappa\sqrt{n}\log(1/\epsilon) )$ &$ O(\kappa\mu/\sqrt{n})$\\
\hline 
FB-splitting  & $O(\kappa^2n\log(1/\epsilon) )$ & $\tilde{O}(\kappa n\log(1/\epsilon) )$ &$ O(\kappa\mu)$\\
\hline 
Ac-FB splitting  & $O(\kappa n\log(1/\epsilon) )$ & no acceleration & no acceleration \\
\hline 
\end{tabular}
\caption{Algorithmic complexity before and after Catalyst acceleration when the problem is ill-conditioned ($\kappa^2\geq n$), Here, $\tilde{O}$ hides a logarithmic factor, and Ac-FB splitting is the accelerated FB splitting proposed in \cite{Balamurugan2016} by assuming a linear $B$.}\label{table: 2}
\end{table}

\section{Asynchronous implementation}\label{sec:asyn}
In this section we study the asynchronous extension of Algorithm \ref{alg:Framwork} for the pure forward step method:
\begin{eqnarray*}
x_{k+1}  = \left(I-\gamma\,B\right)\left(x_{k}\right)=x_k-\gamma B(x_k),\,\forall k \geq 0,
\end{eqnarray*}
where $A \equiv 0$, and the unique solution satisfies $B(x^*)=0$.

Our asynchronous setting is similar to Hogwild! \citep{asgd}, AsySCD \citep{ascda}, and PASSCoDe \citep{asdca}.
We assume a multicore architecture where each core makes updates to a centrally stored vector $x$ in an asynchronous manner to reduce the computation time. 
The framework of the general asynchronous randomized forward-step algorithm is described in Algorithm \ref{alg:asyn}.

\begin{algorithm}[!htb]
\caption{Asynchronous randomized forward-step algorithm }
\begin{algorithmic}[1]\label{alg:asyn}
\REQUIRE ~~\\ 
An initial value $x_0\in \mathbb{R}^d$, an algorithm-specific $\phi_i^0\in \mathbb{R}^d$ for all $i\in [n],$ the total number of iterations $T$.
\STATE \textbf{While} the number of updates $\leq T$ \textbf{do in parallel}
\STATE  ~~Read the central $x$ and the proxies $\phi$ in the central memory
\STATE  ~~Random sampling an integer $I_{k}$ from $\{1,2,\ldots,n \}$, compute $B_{I_{k}}(x)$ 
\STATE  ~~Add $\gamma(B_{I_{k}}(x)-\phi_{I_{k}}+\frac{1}{n}\sum_{i=1}^n \phi_{i})$ to the centrally stored $x$ and let proxies update according to the specific update scheme in Algorithm \ref{alg:Framwork}.
\STATE \textbf{End while}
\ENSURE ~~\\ 
\STATE $x_{T+1}$
\end{algorithmic}
\end{algorithm}

The main steps of Algorithm \ref{alg:asyn} are read, evaluate, and update. 
Following the notation in \citet{reddi2015variance}, we use a global counter $k$ to track the number of updates that are successfully executed to the centrally stored $x$. 
Such an after-write approach has become a standard global labeling scheme in the literature \citep{ asdca,ASAGA,ascda,asgd,reddi2015variance}.
The value of the centrally stored $x$ and $\phi = (\phi_i)_{i \in [n]}$ after $k$ updates are denoted as $x_{k}$ and $\phi^k = (\phi_i^k)_{i \in [n]}$, respectively.
Each processor does the read-evaluate-update steps simultaneously, and so $x$ and $\phi$ can have different time labels in the read and update steps.  
We use $D(k)\in\{1,2,\ldots,k\}$ to denote the time label of the particular $x$ and proxies used in the read step of the $(k+1)$th update. 
This means that $\phi^{D(k)}$ and $x_{D(k)}$ are used to compute the value added to the centrally stored $x$ in the $(k+1)$th update, and we define
$$\hat{x}_k \triangleq x_{D(k)}, \, \hat{\phi}_i^k \triangleq \phi_{i}^{D(k)}, \, k \geq 0,\quad i\in [n].$$
We make the following assumption for the convergence analysis of the asynchronous Algorithm \ref{alg:asyn}.
\begin{assumption}\label{assum3}
~~~
\begin{itemize}
\item[(\ref{assum3}.1)] There exists an integer $\tau \geq 0$ such that $0 \leq k-D(k) \leq \tau$ for all $k\geq 0$. 
\item[(\ref{assum3}.2)] There exists a constant $M_{\phi,B}\geq0$ such that $\|\phi_i^k\| \leq M_{\phi,B}$ and $\|B_i(x_k)\|\leq M_{\phi,B}$ for all $k\geq 0$ and $i \in [n]$.
\end{itemize}
\end{assumption}
The first assumption bounds the delay between the read and update times by $\tau$.
This assumption is typical in asynchronous systems \citep{mania2015,reddi2015variance}. 
Following \citet{mania2015}, we also assume that $\|B_{i}(x_{k})\|$ and $\|\phi_i^k\|$ for all $i \in [n]$ are bounded for all $k\geq 0$.

The asynchronous variance-reduced forward step can be analyzed with the perturbed iterate framework from \citet{mania2015} which has the following form
\begin{align}\label{asyn1}
x_{k+1}=x_{k}-\gamma\, \mathcal{G}(\hat{x}_k,\hat{\phi}^k, I_{k}) ,~& k\geq 0,
\end{align}
where
\begin{align*}
\mathcal{G}(\hat{x}_k,\hat{\phi}^k, I_{k})\triangleq B_{I_{k}}\left(\hat{x}_k\right)-\hat{\phi}^k_{I_k}+\frac{1}{n}\sum_{i=1}^{n}\hat{\phi}^k_{i},~~ k\geq 0.
\end{align*}
Here, $\hat{x}_k$ and $\hat{\phi}^k$ can be interpreted as perturbed versions of $x_k$ and $\phi^k$ due to asynchronous updating.
It is readily seen that $\mathbb{E}\left[ \mathcal{G}(\hat{x}_k,\hat{\phi}^k, I_{k})\,|\,\mathcal{F}_k  \right]=B(\hat{x}_k)$, and so this estimator is unbiased.
Applying \eqref{asyn1} recursively shows that
\begin{align}\label{asyn2.5}
x_{k}=\overbrace{\underbrace{x_{0}-\gamma\, \mathcal{G}(\hat{x}_0,\hat{\phi}^0,I_{0})}_{x_1}-\gamma\, \mathcal{G}(\hat{x}_1,\hat{\phi}^1,I_{1})}^{x_2}-\cdots-\gamma\, \mathcal{G}(\hat{x}_{k-1},\hat{\phi}^{k-1},I_{k-1}),\,\forall k\geq 1.
\end{align}
Since $\hat{x}_k=x_k$ when $\tau=0$, the above scheme recovers the synchronous case when $\tau=0$. 

The following lemma is the key to our analysis of Algorithm \ref{alg:asyn}. 
In this lemma, we show how the expected distance to the optimal solution $x^*$ changes after one iteration.
Similar to Lemma \ref{lemma1}, we show that the expected distance to $x^*$ contracts by a factor strictly less than one plus an additional disturbance.
This additional disturbance depends on the variance of the estimator and the error due to the asynchronous updates.

\begin{lemma}\label{asyn3}
Suppose Assumptions \ref{assum1} and \ref{assum3} hold, and let $\{ x_k  \}_{k\geq 0}$ be generated by Algorithm \ref{alg:asyn}. Then, for all $k\geq 0$ we have
\begin{align}\label{asyn3.5}
 \mathbb{E}\|x_{k+1}-x^*\|^2 &\leq (1-\gamma\mu) \mathbb{E}\|x_{k}-x^*\|^2+2\gamma\mu  \mathbb{E}\|\hat{x}_k -x_{k}\|^2\\&~~~+\gamma^2 \mathbb{E}\|\mathcal{G}(\hat{x}_k,\hat{\phi}^k, I_{k})\|^2+2\gamma  \mathbb{E}\langle\mathcal{G}(\hat{x}_k,\hat{\phi}^k, I_{k}),\hat{x}_k-x_{k}\rangle.\notag
\end{align}
\end{lemma}

Eq. \eqref{asyn3.5} shows that the convergence of Algorithm \ref{alg:asyn} is closely related to three error terms: the distance between the true $x_k$ and its perturbed version $\hat{x}_k$, the norm of the estimator $\mathcal{G}(\hat{x}_k,\hat{\phi}^k, I_{k})$, and the inner product of the mismatch $x_k-\hat{x}_k$ and the estimator $ \mathcal{G}(\hat{x}_k,\hat{\phi}^k, I_{k})$. The following lemma bounds the three error terms using the constant $M \triangleq 3M_{\phi, B}$.

\begin{lemma}\label{asyn4}
Suppose Assumption \ref{assum1} and Assumption \ref{assum3} hold, and let $\{ x_k  \}_{k\geq 0}$ be generated by Algorithm (\ref{alg:asyn}). Then, for all $k\geq 0$ we have
\[
\mathbb{E} \|x_{k}-\hat{x}_k\|^2 \leq \gamma^2\tau^2M^2, \quad \mathbb{E}\langle\mathcal{G}(\hat{x}_k,\hat{\phi}^k, I_{k}),\hat{x}_{k}-x_{k}\rangle\leq \gamma\tau M^2,
\]
and
\begin{align*}
\mathbb{E}\|\mathcal{G}(\hat{x}_k,\hat{\phi}^k, I_{k})\|^2&\leq 6L^2\gamma^2\tau^2M^2 +6L^2 \mathbb{E}\|x_{k}-x^* \|^2 +\frac{4}{n}\mathbb{E}\sum_{i=1}^n\|\phi_{i}^k-B_{i}(x^*)\|^2\\
&~~+\frac{4}{n}\mathbb{E}\sum_{i=1}^n\|\phi_{i}^k-\hat{\phi}_{i}^{k}\|^2. 
\end{align*}
\end{lemma}

We define an additional constant
$$
\mathcal{E}_{0}\triangleq 2\gamma^3\mu\tau^2M^2+6\gamma^4L^2\tau^2M^2+2\gamma^2\tau M^2,
$$ 
and the sequence $a_{k}\triangleq  \mathbb{E}\|x_{k}-x^*\|^2$ for all $k \geq 0$ to succinctly express Lemmas \ref{asyn3} and \ref{asyn4} in the following corollary.
\begin{corollary}\label{asyn_coro}
Suppose Assumptions \ref{assum1} and \ref{assum3} hold, and let $\{ x_k  \}_{k\geq 0}$ be generated by Algorithm \ref{alg:asyn}. Then, for all $k\geq 0$ we have
\begin{equation}\label{eq:asyn main}
a_{k+1}\leq(1-\gamma\mu+6\gamma^2L^2)a_{k}+\frac{4\gamma^2}{n}\ \mathbb{E}\sum_{i=1}^n\|\phi_{i}^k-B_{i}(x^*)\|^2+\frac{4\gamma^2}{n} \mathbb{E}\sum_{i=1}^n\|\phi_{i}^k-\hat{\phi}_{i}^{k}\|^2+\mathcal{E}_{0}. 
\end{equation}
\end{corollary}
Corollary \ref{asyn_coro} is similar to Corollary \ref{coro1}.
However, (\ref{eq:asyn main}) includes two additional terms from the error caused by asynchronous updates. 

To continue, we introduce the following assumptions on $G(\phi^k)$ and $H(\phi^k)$.
We introduce two new non-negative parameters $\mathcal{E}_1$ and $\mathcal{E}_2$ to capture the error caused by the asynchronous updates. We show how to choose these later for specific algorithms.
\begin{assumption}\label{assum4}
~~~
\begin{itemize}
\item[(\ref{assum4}.1)] There exist constants $c_1$, $c_2$, and $\mathcal{E}_1$ with $0\leq c_1<1$, $c_2\geq0$, and $\mathcal{E}_1\geq 0$, 
such that $\mathbb{E} G(\phi^{k+1}) \leq c_1 \mathbb{E} G(\phi^k) +c_2 \mathbb{E} \|x_{k}-x^{*}\|^{2}+\mathcal{E}_1.$ 
\item[(\ref{assum4}.2)] There exist a constant $c_3$ such that $H(\phi^{k}) \leq  c_3  L_{\rho}(x_{S_i},\phi^{S_i})$ for all $ S_i\leq k < S_{i+1}$.
\item[(\ref{assum4}.3)] There exists a constant $\mathcal{E}_2\geq 0$ such that $\sum_{i=1}^n \mathbb{E}\|\phi_{i}^k-\hat{\phi}_{i}^{k}\|^2/n\leq \mathcal{E}_2.$ 
\item[(\ref{assum4}.4)] The system is fully synchronized at the end of every epoch.
\end{itemize}
\end{assumption}

The above assumption is similar to Assumption \ref{assum2} with the additional requirement that the average expected distance between $\hat{\phi}_i^k$ and $\phi_i^k$ is bounded by a constant. Assumption (\ref{assum4}.4) also appears in \citet{reddi2015variance}. Then, we define
\begin{align*}
\theta&\triangleq \max \left\lbrace 1-\gamma\mu+6\gamma^2L^2+c_2\rho,\frac{4\gamma^2}{\rho}+c_1  \right\rbrace ,\\
\lambda&\triangleq\theta+4\gamma^2\bar{m}c_3,
\end{align*}
and
\begin{align*}
\mathcal{E}_3\triangleq 4\gamma^2\mathcal{E}_2+\mathcal{E}_0+\rho\mathcal{E}_1,~~~~~~~~~~~~~~~~~~~~~~~~~~~~~~~
\end{align*}
where we recall $\bar{m}$ is the upper bound on the epoch lengths from Assumption \ref{assum2}. We have the following step-size rule:
\begin{equation}\label{eq:asyn gamma ineq}
\gamma<\min\left\lbrace \left(\frac{\mu}{\frac{3(1-c_1)L^2}{2}+(1-c_1)\bar{m}c_3+c_2}\right)^2 ,\left(\frac{1-c_1}{4+4\bar{m}(\frac{1-c_1}{4})^3c_3}\right)^2\right\rbrace.
\end{equation}

\begin{theorem}\label{thm:asyn}
Suppose Assumptions \ref{assum1}, \ref{assum3}, and \ref{assum4} hold, and that the step-size $\gamma$ satisfies \eqref{eq:asyn gamma ineq}. Then, for $\rho=\gamma^{1.5}$, we have $\theta,\, \lambda\in[0,1)$, and
\begin{equation}\label{eq:asyn L_rho contraction}
\mathbb{E}L_{\rho}(\tilde{x}_{k},\tilde{\phi}^{k})\leq \lambda\,  \mathbb{E}L_{\rho}(\tilde{x}_{k-1},\tilde{\phi}^{k-1})+\frac{\mathcal{E}_3}{1-\theta},\,\forall k \geq 1.
\end{equation} 

\end{theorem}

Theorem \ref{thm:asyn} shows that the iterates generated by Algorithm \ref{alg:asyn} converge linearly to some neighborhood of $x^*$. We will see in our examples that if the delay $\tau=0$, then the parameters $\mathcal{E}_0$, $\mathcal{E}_1$, $\mathcal{E}_2$, and $\mathcal{E}_3$ are all zero, so that Theorem \ref{thm:asyn} matches the synchronous case.

\section{Examples}\label{sec:examples}

This section considers several specific algorithms within our framework.
For these algorithms, we first verify Assumption \ref{assum2}, then we analyze the number of operator evaluations required to find $x_{\epsilon}^*$ such that $\mathbb{E}\| x_{\epsilon}^*-x^* \|^2\leq \epsilon$.
We also compute the complexity of Catalyst, and verify Assumption \ref{assum4} for asynchronous implementation where applicable.
The detailed computations are in Appendix \ref{appendix:example}.
Throughout, we let $\{W_k\}_{k\geq 0}$ be a sequence of i.i.d. standard uniform random variables. We also use the common notation $\tilde{O}$ to hide logarithmic factors \citep{Catalyst,catalyst_2}.

\subsection{GD}\label{ex_gd}
Full gradient descent (see \citep{Nesterov:2014:ILC:2670022}) follows \eqref{eq:randomFB} where the proxies satisfy $\phi_i^{k}=B_i\left( x_{k} \right)$ for all $i \in [n]$, and so $\mathcal{G}\left(x_{k},\,\phi^k,I_{k}\right)=\frac{1}{n}\sum_{i=1}^nB_i\left(x_k\right)=B(x_k)$ for all $k \geq 0$. We choose $\mathcal{S}=\emptyset$ and so
$$ 
H(\phi^k)=\frac{1}{n}\| \phi_i^k-B_i(x^*) \|^2=\frac{1}{n}\sum_{i=1}^n\| B_i(x_k)-B_i(x^*) \|^2\leq L^2\| x_k-x^* \|.
$$
It is then readily seen that Assumption \ref{assum2} holds with $c_1=c_2=0$, $c_3=L^2$, and $m_i=1$ for all $i\geq1$.

Since $\mathcal{G}\left(x_{k},\,\phi^k,I_{k}\right) = B(x_k)$, we choose constant  step-size  $\gamma = \mu/L^2$ to minimize the coefficient in \eqref{eq:stochastic-1} and then use Lemma \ref{lemma1} to obtain
\begin{align*}
\| x_k-x^* \|^2\leq \left(1-\frac{1}{\kappa^2}\right)\|x_{k-1}-x^*  \|^2,\,\forall k \geq 1.
\end{align*}
Thus, GD requires $O(\kappa^2\log(1/\epsilon))$ iterations to obtain an $\epsilon$-solution, and each iteration requires $n$ operator evaluations,
The total complexity is $O(n\,\kappa^2\log(1/\epsilon))$, i.e., $\pi_{\mathcal{M}(A,B)}=n\,\kappa^2.$
Substitute this $\pi_{\mathcal{M}(A,B)}$ value to \eqref{prop:compl of cata} to see the complexity of GD with Catalyst is
\begin{align*}
\tilde{O}\left( n\frac{\sigma+\mu}{\mu}\left( \frac{L+\mu}{\sigma+\mu} \right)^2\log\left(\frac{1}{\epsilon}\right) \right).
\end{align*}
For $ \sigma=(\kappa-1)\mu$, the complexity of GD with Catalyst is $\tilde{O}(n\,\kappa\log(1/\epsilon) )$.
This is the same as the complexity of the accelerated Forward-Backward algorithm from \cite{Balamurugan2016} (up to a logarithmic factor).
However, \cite{Balamurugan2016}  requires $B$ to be \emph{linear}, while our analysis holds for general $B$ satisfying Assumption \ref{assum1}.

\subsection{SVRG}\label{ex:svrg}
SVRG \citep{JohnsonSVRG,S2GD,Lin:SVRG} is epoch-based. Here we allow  unequal epoch lengths as long as $\bar{m}<\infty$. SVRG follows \eqref{eq:randomFB} where the proxies satisfy:
\begin{equation*}
\phi_i^{t}= B_i(x_{S_k}),\, \forall i \in [n],\,  S_k \leq t < S_{k+1}.
\end{equation*}
We select $\mathcal{S}=\emptyset$ and thus Assumption \ref{assum2}.1 holds with $c_1=c_2=0$. Since $\phi_{i}^t=B_i(x_{S_k})$ for all $i \in [n]$ and $S_k\leq t<S_{k+1}$, we have 
\begin{align*}
H(\phi^t)&=\frac{1}{n}\sum_{i=1}^n \| B_i(x_{S_k})-B_i(x^*) \|^2
\leq L^2  \| x_{S_k}-x^* \|^2
\leq L^2 L_{\rho}\left( x_{S_k},\phi^{S_k}  \right),
\end{align*}
where the first inequality is due to $L$-Lipschitz continuity of each $B_i$,
and the second inequality is by definition of $L_\rho$. 
Therefore, Assumption \ref{assum2}.2 holds with $c_3=L^2$. 

We consider the complexity of SVRG for constant epoch length $m \geq 1$ and constant step-size $\gamma > 0$.
We choose step-size $\gamma=\mu/3L^2$ and epoch length $m=O(\kappa^2)$ so that the expected distance to $x^*$ contracts by a factor of at least $3/4$ after every epoch. 
To achieve an $\epsilon-$solution, we need $O(\log(1/\epsilon))$ epochs. 
In the beginning of each epoch, the proxy update requires $n$ operator evaluations, and each of the inner $m$ iterations requires two operator evaluations. 
With $m=O(\kappa^2)$, the total number of operator evaluations is $O((\kappa^2+n)\log(1/\epsilon))$. 
In the ill-conditioned setting with $n\leq \kappa^2$ \citep{Catalyst}, the total complexity is $O\left(\kappa^2 \log(1/\epsilon) \right)$.

Now we consider SVRG with Catalyst for $n\leq \kappa^2$ where $\pi_{\mathcal{M}(A,B)}=n+\kappa^2$ and $\pi_{\mathcal{M}(A,B+\sigma I)}=n+\left((L+\sigma)/(\mu+\sigma)\right)^2$.
Choosing $\sigma^*=O(\kappa\mu/\sqrt{n})$ minimizes the complexity of \eqref{prop:compl of cata}, and so the complexity of SVRG with Catalyst is $O(\kappa\sqrt{n}\log\left(1/\epsilon\right))$.
 This is less than the complexity of SVRG (i.e., $O(\kappa^2\log\left(1/\epsilon\right))$) when $\kappa^2\geq n$.

In asynchronous SVRG, the system is synchronized after every epoch so we have $ S_k\leq D(t)\leq t$ for all $t\geq S_k$ \citep[Section 3]{reddi2015variance}.
The proxies do not change within an epoch and so $\phi^k=\phi^{D(k)}=\hat{\phi}^k$. 
As a result, Assumption \ref{assum4}.3 holds with $\mathcal{E}_2=0$.
To verify Assumption \ref{assum4}.1, we have $G(\phi^k)\equiv 0$ since $\mathcal{S}=\emptyset$.
Assumption \ref{assum4}.2 is the same as Assumption \ref{assum2}.2.
So Assumption \ref{assum4} holds with $\mathcal{S}=\emptyset$, $c_1=c_2=\mathcal{E}_1=\mathcal{E}_2=0$, and $c_3=L^2$.

\subsection{SAGA}\label{ex:saga}
SAGA \citep{DefazioSAGA} is initialized with $\phi_i^0 = B_i(x_0)$ for all $i \in [n]$ and the proxies are updated according to:
\begin{align*}
\phi_i^{k+1}=
\begin{cases}
B_{i}(x_k), &  i=I_k,\\
\phi_{i}^k, & i \neq I_k.
\end{cases}
\end{align*}
Let $\mathcal{S}=[n]$ and $m_j=1$ for all $j \geq 1$. 
To verify Assumption \ref{assum2}.1, observe that $\phi_i^{k}$ changes to $B_i(x_{k-1})$ only when $I_{k-1}=i$, and so  for all $k\geq 1$ we have
\begin{align*}
\mathbb{E}\left[G(\phi^k)|\mathcal{F}_{k-1}\right]=&\frac{1}{n}\sum_{i=1}^{n} \left(G(\phi^{k-1})+\frac{1}{n}\left(\|(B_{i}(x_{k-1})-B_{i}(x^{*})\|^{2}-\|\phi_{i}^{k-1}-B_{i}(x^{*})\|^{2}\right)\right)\\
=&\left(1-\frac{1}{n}\right)G(\phi^{k-1})+\frac{1}{n}\sum_{i=1}^{n} \frac{\|B_{i}(x_{k-1})-B_{i}(x^{*}) \|^{2}}{n}.
\end{align*}
Since all $\{B_{i}\}_{i \in [n]}$ are $L$-Lipschitz, we have
\begin{align} \label{contract of G for saga}
\mathbb{E}G(\phi^k)\leq  (1-\frac{1}{n})\mathbb{E}G(\phi^{k-1})+\frac{L^2}{n} \mathbb{E}\|x_{k-1}-x^{*}\|^{2},~~k\geq 1.
\end{align}
We conclude that Assumption \ref{assum2}.1 holds with $c_1=1-1/n$ and $c_2=L^2/n.$ Furthermore, Assumption \ref{assum2}.2 holds with $c_3=0$ since $H\left( \phi^k \right)\equiv0.$  
When $\gamma=O(\mu/L^2)$, $\pi_{\mathcal{M}(A,B)}=\max \left\{n,\kappa^2 \right\}$.
If $n\leq \kappa^2$, then this complexity is $O\left(\kappa^2 \log(1/\epsilon)\right)$ (the same as SVRG).

To analyze SAGA with Catalyst, recall $\pi_{\mathcal{M}(A,B)}=\max \left\{n,\kappa^2 \right\}$ and
\begin{align*}
\pi_{\mathcal{M},A,B+\sigma (I-\bar{x})} = \max\left\{n,\left(\frac{L+\sigma}{\mu+\sigma}\right)^2\right\}.
\end{align*}
Choose $\sigma^*=O(\kappa\mu/\sqrt{n})$ to minimize \eqref{prop:compl of cata} and so the overall complexity of SAGA with Catalyst is $O(\kappa\sqrt{n}\log\left(1/\epsilon\right))$, which is smaller than the complexity of classical SAGA under the condition $\kappa^2\geq n$.

In asynchronous SAGA, Assumption \ref{assum4}.2 holds with $c_3=0$ since $H\left( \phi^k \right)\equiv0.$
To verify Assumption \ref{assum4}.3, observe that SAGA only changes one proxy in each iteration, and so there are at most $\tau$ terms among all $\{\hat{\phi}_i^k\}_{i=1}^n$ that are different from their counterparts in $\{\phi_i^k\}_{i=1}^n$. 
Use Assumption \ref{assum3}.2 to see that $\sum_{i=1}^n \mathbb{E}\|\phi_{i}^k-\hat{\phi}_{i}^{k}\|^2/n \leq 4\tau M_{\phi,B}^2/n$. Finally, Assumption \ref{assum4} holds with $\mathcal{S}=[n]$, $m_j=1$ for $j\geq1$, $\mathcal{E}_1=2L^2\gamma^2\tau^2M^2/n$, $c_1=1-1/n$, $c_2=2L^2/n$, $c_3=0$, and $\mathcal{E}_2= 4\tau M_{\phi,B}^2/n$.

\subsection{SVRG-rand}\label{section:svrgr}
Empirically, SVRG tends to be slower than SAGA given the same computational budget, 
while SAGA requires greater storage cost \citep{DefazioSAGA,reddi2015variance}. 
In this section, we propose a new algorithm based on SVRG which we call SVRG-rand that allows for random epoch lengths. SVRG-rand does not require additional storage cost,
and our simulations suggest that it performs better than classical SVRG.

We introduce the sequence $\{p_k\}_{k \geq 0}$ where $p_k\in[0,1]$ is the probability of doing a full proxy update in iteration $k \geq 0$.
We also define $\underline{p} \triangleq \inf_{k\geq 0}\{p_k\}$ and $ \overline{p} \triangleq \sup_{k\geq 0}\{p_k\}$. Accordingly, the proxies are updated according to:
\begin{align}\label{alg:svrgr}
\phi_{i}^{k+1}=\mathbb{I}(W_k \leq p_k)B_{i}(x_k)+\mathbb{I}(W_k >p_k)\phi_i^k, \, i \in [n].
\end{align}
We initialize with $\phi_i^0=0$ for all $i \in [n]$ which differs from the initialization in SVRG. As long as $\underline{p} >0$, SVRG-rand satisfies Assumption \ref{assum2} with  $\mathcal{S}= [n]$, $m_i=1$ for all $i \geq 1$, $c_1=1-\underline{p}$, and $c_2=\overline{p}L^2$ (see Theorem \ref{thm:svrgr}).
SVRG-rand is essentially not epoch-based, nevertheless as an extension of SVRG it does not require additional storage cost.
Classical SVRG is recovered as a special case of SVRG-rand 
if $p_k=1$ when $k$ is a multiple of $m$, and $p_k = 0$ otherwise. 
SVRG with increasing epochs can be recovered similarly (see \citep{svrg++}).

We analyze the complexity of SVRG-rand where $p_k = 1/n$ for all $k\geq 0$. 
Then $\overline{p}=\underline{p} = 1/n$, and \eqref{contract of G for saga} holds for SVRG-rand by Theorem \ref{thm:svrgr}.
The rest of the complexity analysis is then the same as SAGA and we achieve a complexity of $O(\max\{ n,\kappa^2 \}\log(1/\epsilon))$.
If $n\leq \kappa^2$, then this complexity is $O\left(\kappa^2 \log(1/\epsilon)\right)$ (the same as SVRG and SAGA).
Based on \eqref{prop:compl of cata}, the complexity of SVRG-rand with Catalyst is also the same as SAGA and SVRG, i.e., $O(\kappa\sqrt{n}\log\left(1/\epsilon\right))$.

We now consider asynchronous SVRG-rand. Assumption \ref{assum4}.2 holds with $c_3=0$ since $H\left( \phi^k \right)\equiv0.$
Because the system is synchronized after every epoch, $\phi^k=\phi^{D(k)}=\hat{\phi}^k$ as in SVRG. 
As a result, Assumption \ref{assum4}.3 holds with $\mathcal{E}_2=0$.
Similar to  Theorem \ref{thm:svrgr}, we have
$$
\mathbb{E}G(\phi^k)\leq (1-\underline{p}) \mathbb{E}G(\phi^{k-1})+\overline{p}L^2  \mathbb{E}\|\hat{x}_{k-1}-x^*\|^2,\,\forall k \geq 1.
$$
By adding and subtracting $x_{k-1}$ in the second term of the above inequality, we have
$$
\mathbb{E}G(\phi^k)\leq (1-\underline{p}) \mathbb{E}G(\phi^{k-1})+2\overline{p}L^2\mathbb{E}\|x_{k-1}-x^*\|^2+2\overline{p}L^2\gamma^2\tau^2M^2,\,\forall k \geq 1.
$$
Therefore, Assumption \ref{assum4} holds with $\mathcal{S} = [n]$, $m_j=1$ for all $j\geq1$, $c_1=1-\underline{p}$, $c_2=2\overline{p}L^2$, $\mathcal{E}_1=2\overline{p}L^2\gamma^2\tau^2M^2$, $c_3=0$, and $\mathcal{E}_2=0$.

\subsection{SAGD}

SAGD \citep{SAGD} is based on a probabilistic interpolation between GD and SAGA.
In each iteration, SAGD takes a GD step with probability $q$ and a SAGA step with probability $1-q$. The proxies are then updated according to: 
\begin{align}\label{eq:SAGD proxy}
\phi_i^{k+1}=\left\{\mathbb{I}(W_k\leq q)+\mathbb{I}(W_k> q)\mathbb{I}(I_k=i)  \right\}B_i(x_k)+\mathbb{I}(W_k>q)\mathbb{I}(I_k\neq i)\phi_{i}^k.
\end{align}
That is, SAGD updates all the proxies in a GD step and it updates only one proxy in a SAGA step.
We initialize $\phi_i^0=0$ for all $i \in [n]$ as in SVRG-rand. Then, we choose $\mathcal{S} = [n]$ and compute:
\begin{align}
\mathbb{E}\left[ G(\phi^{k+1})\Big|\mathcal{F}_{k}\right] 
&=\left(q+\frac{1-q}{n}\right)\frac{1}{n}\sum_{i=1}^n \| B_{i}(x_{k})-B_{i}(x^*)\|^2+(1-q)\left(1-\frac{1}{n}\right)G(\phi^k)\nonumber\\
&\leq (1-q)\left(1-\frac{1}{n}\right)G(\phi^k)+\left(q+\frac{1-q}{n}\right)L^2 \| x_k-x^* \|.\label{contract of G sagd}
\end{align}
Therefore, SAGD satisfies Assumption \ref{assum2} with $\mathcal{S}=[n]$, $c_1=(1-q)\left(1-1/n\right)$, $c_2=1-c_1$, and $c_3=0$.

We consider the complexity of SAGD for $q=1/n$ (the analysis for general $q$ is similar). When $q=1/n$, \eqref{contract of G sagd} implies
\begin{align*}
\mathbb{E}G(\phi^k)\leq  (1-\frac{1}{n})\mathbb{E}G(\phi^{k-1})+\frac{2L^2}{n} \mathbb{E}\|x_{k-1}-x^{*}\|^{2},\, \forall k\geq 1.
\end{align*}
The above display is almost the same as \eqref{contract of G for saga} for SAGA.
We can thus use the same arguments to see the complexity of SAGD under $q=1/n$ is $O(\max\{ n,\kappa^2 \}\log(1/\epsilon))$. 
When $\kappa^2\geq n$, the complexity of SAGD with Catalyst is $O(\kappa\sqrt{n}\log(1/\epsilon))$.

\subsection{HSAG}\label{ex:hsag}
HSAG \citep{reddi2015variance} combines SVRG and SAGA. For a set $\mathcal{S}\subset[n]$ with cardinality $S\triangleq| \mathcal{S} |$, the proxies in HSAG are updated according to:
$$
\phi_{i}^{k+1}=
\begin{cases}
\mathbb{I}\left(I_{k}=i\right)B_{i}\left(x_{k}\right)+\mathbb{I}\left(I_{k}\ne i\right)\phi_{i}^{k}, & i\in \mathcal{S},\\
\mathbb{I}\left(m\,\mid\,(k+1)\right)B_{i}\left(x_{k+1}\right)+\mathbb{I}\left(m\,\nmid\,(k+1)\right)\phi_{i}^{k}, & i\in \mathcal{S}^{c}.
\end{cases} 
$$
Then, Assumption \ref{assum2}.1 holds with $c_1=1-1/n$ and $c_2=SL^2/n^2$, and Assumption \ref{assum2}.2 holds with $c_3=\left(n-S  \right)L^2/n $ and $m_j=m$ for all $j\geq 1$.

We choose $m=O\left(\max\{\kappa^2,n\mathbb{I}(S>0)\} \right)$ to obtain a complexity of
$$
\pi_{\mathcal{M}(A,B)}=O\left( \max\{\kappa^2,n\mathbb{I}(S>0)\}+n-S\right)\log(1/\epsilon).
$$
This expression recovers the complexity of SVRG when $S=0$ and SAGA when $S=n$. 
Choose $\sigma^*=O(\kappa\mu/\sqrt{n})$ to minimize \eqref{prop:compl of cata} and so
the complexity of HSAG with Catalyst is $O\left(\kappa\sqrt{n}\log\left(1/\epsilon\right)\right)$ for $S=O(n)$.

For asynchronous HSAG, Assumption \ref{assum4} holds with $c_1=1-1/n$, $c_2=2SL^2/n^2$, $\mathcal{E}_1=2SL^2\gamma^2\tau^2M^2/n^2$, $c_3=\left(n-S  \right)L^2/n$, and $\mathcal{E}_2=4\tau M_{\phi,B}^2/n$.

\subsection{SAGA+SVRG-rand}\label{section:saga+svrgr}
We build on the hybridization idea of HSAG and combine SVRG-rand and SAGA. SVRG-rand needs to compute the full operator with a positive, though usually small, probability in every iteration. 
The full operator evaluation is time-consuming for large $n$.
Alternatively, SAGA achieves variance reduction by storing the entire proxy vector, which avoids the need for periodic full operator evaluations. 

Let $\mathcal{S}_1 \subseteq [n]$ be the index set for the proxies following a SAGA-type update, so the proxies in $\mathcal{S}_2=\mathcal{S}_1^c$ follow SVRG-rand updates. Then, the proxies are updated according to:
\begin{align}\label{svrg+saga}
\phi_{i}^{k+1}=\begin{cases}
\mathbb{I}\left(I_{k}=i\right)B_{i}\left(x_{k}\right)+\mathbb{I}\left(I_{k}\ne i\right)\phi_{i}^{k}, & i\in \mathcal{S}_1,\\
\mathbb{I}(W_k \leq p_k)B_{i}(x_k)+\mathbb{I}(W_k >p_k)\phi_i^k & i\in \mathcal{S}_2.
\end{cases} 
\end{align}
 Under the same assumptions on $\{p_k\}_{k \geq 1}$ as SVRG-rand, SAGA+SVRG-rand satisfies Assumption \ref{assum2} with $\mathcal{S} = [n]$, $m_i=1$ for all $i\geq1$, 
$c_3=0$, $c_1=\max\left\lbrace 1-1/n,1-\underline{p}  \right\rbrace$, and $c_2= S_1 L^2/n^2+\overline{p}S_2 L^2/n$.

The complexity of SAGA+SVRG-rand when $p_k=1/n$ for all $k\geq 1$ is $O(\max\{ n,\kappa^2 \}\log(1/\epsilon))$, which is the same as SAGA.
In view of this observation, the complexity of SAGA+SVRG-rand with Catalyst is  $O(\kappa\sqrt{n}\log\left(1/\epsilon\right))$.

\subsection{SARAH}

SARAH presented in Algorithm \ref{alg:SARAH} is another extension of SVRG \citep{SARAH}.  
SARAH is also epoch-based, but it uses a biased gradient estimator $v_k$.
Unless $k$ is a multiple of the epoch length $m$, we have
\begin{align*}
\mathbb{E}\left[ v_k \Big | \mathcal{F}_{k} \right]=B(x_k)-B(x_{k-1})+v_{k-1}\neq B(x_k).
\end{align*}

\begin{algorithm}[htb] 
\caption{SARAH }
\label{alg:SARAH} 
\begin{algorithmic}[1] 
\REQUIRE ~~\\ 
An initial value $x_0\in \mathbb{R}^d$,   the total number of iterations $T$, and the length of epochs $m$. Constant step-size $\gamma$.
\FOR{each $k =0,1,\ldots,T$}
\STATE Generate $I_k$ uniformly in $[n]$.
\STATE  $v_k=\begin{cases}\frac{1}{n}\sum_{i=1}^n B_{i}(x_k), & k\mid m \\  B_{I_k}(x_k)-B_{I_k}(x_{k-1})+v_{k-1}, & k\nmid m \end{cases}$
\STATE $x_{k+1}=x_k-\gamma v_k$
\ENDFOR
\ENSURE ~~\\ 
\STATE $x_{T+1}$
\end{algorithmic}
\end{algorithm}

Recall $\tilde{x}_k=x_{km}$ for $k\geq 0$ are the iterates at the beginning of each epoch.
In Appendix E, we establish linear convergence of SARAH for solving Problem \eqref{eq:inclusion} when $A \equiv 0$ by showing that $\mathbb{E}\| B(\tilde{x}_k) \|^2$ contracts by a factor of $3/4$ after every epoch, when the epoch length is $m=O(\kappa^2)$.
The overall complexity of SARAH is then $O\left((n+\kappa^2)\mbox{log}(1/\epsilon)\right)$.

\subsection{Summary}

To sum up, we summarize the complexity of the algorithms considered in this section. 
Using the fact
$
\frac{1}{2}(n+\kappa^2)\leq \max\{n,\kappa^2 \}\leq n+\kappa^2
$,
it is readily seem that $\pi_{\mathcal{M}(A,B)}$ for SAGA (and also several other examples in this section) can  be rewritten as $n+\kappa^2$.
The complexity of solving  Problem \eqref{eq:inclusion} is larger since Problem \eqref{eq:inclusion} covers a larger class of problems (e.g. function minimization, saddle-point problems, variational inequalities, etc.).
A similar  phenomenon is observed in \cite{chen2017accelerated} who discussed complexities for solving variational inequality problems.
Since SVRG-rand and SAGA+SVRG-rand are proposed in this study for Problem \eqref{eq:inclusion}, 
there are no existing results on their complexity for Problem \eqref{eq:finite sum}. 
 \begin{table}[!htbp]
\begin{tabular}{|c|c|c|}
\hline 
Algorithm & Complexity of solving  \eqref{eq:inclusion} & Complexity of solving \eqref{eq:finite sum} \\
\hline 
GD &$O(n\kappa^2)\,\log(1/\epsilon) $  & $O(n\kappa)\,\log(1/\epsilon) $ \citep{Nesterov:2014:ILC:2670022}\\
\hline 
SVRG & $O(n+\kappa^2)\log(1/\epsilon) $ & $O(n+\kappa)\,\log(1/\epsilon) $ \citep{JohnsonSVRG}\\
\hline  
SAGA & $O(n+\kappa^2)\log(1/\epsilon) $ & $O(n+\kappa)\,\log(1/\epsilon) $ \citep{DefazioSAGA} \\
\hline
SVRG-rand &  $O(n+\kappa^2)\log(1/\epsilon) $   & no results\\
\hline
SAGD& $O(n+\kappa^2)\log(1/\epsilon) $ & $O(n+\kappa)\,\log(1/\epsilon) $ \citep{SAGD}\\
\hline 
HSAG & $O(n+\kappa^2)\log(1/\epsilon) $ & $O(n+\kappa)\,\log(1/\epsilon) $ \citep{reddi2015variance}\\
\hline
SAGA+SVRG-rand &  $O(n+\kappa^2)\log(1/\epsilon) $   & no results\\
\hline
SARAH &  $O(n+\kappa^2)\log(1/\epsilon) $ & $O(n+\kappa)\,\log(1/\epsilon) $ \citep{SARAH}\\
\hline
\end{tabular}
\caption{Comparison of complexity of different algorithms for solving monotone inclusions and for minimizing finite sums, respectively.}\label{table: 1}
\end{table}

\section{Numerical experiments}\label{sec:numerical}

This section reports numerical experiments for a saddle-point problem (for policy evaluation) and a variational inequality problem (for a strongly monotone two-player game). Both of these applications fall outside function minimization.

We compare the performance of the following algorithms in this section:
\begin{itemize}
\item \textbf{SVRG}.
The epoch length is set to be $m=2n$. \citep{svrg++,JohnsonSVRG,Lin:SVRG}.  
\item \textbf{SAGA}. Following the discussion in Section \ref{sec:examples}, we use a constant step-size $\mu/(7L^2)$.
\item \textbf{SVRG++}. An algorithm of \citet{svrg++} which falls in the category of SVRG with increasing epochs. 
The epoch length is initialized as $2n$ and it is doubled after every epoch \citep{svrg++}. 
\item \textbf{SVRG-rand}. The proposed extension of SVRG described by (\ref{alg:svrgr}). 
We let $p_k$ be a small number for the first $n$ iterations to avoid too many full operator evaluations in early iterations. 
After the first $n$ iterations, $p_k$ is assigned to decrease by half from $1/(2n)$ after every epoch to control the frequency of full updates in later iterations. 
We adopt the common strategy of automatically terminating the epoch if the current epoch is longer than $8n$.
 \citep{svrg++,svrg_adapt}. 
\item \textbf{Hybrid SAGA and SVRG-rand}.
 The proposed hybrid method with update rule given by \eqref{svrg+saga} where the first $n/2$ proxies perform a SAGA-type update. 
 \item \textbf{SAGD.} The probability of performing a GD update is $q=1/(2n)$.
\item \textbf{SARAH.} Following \cite{SARAH}, the epoch length is set as $2n$.
\end{itemize}
We run each algorithm 10 times with the same constant step-size and report the mean performance.
Following the tradition established in \citep{svrg++,DefazioSAGA,Schmidt2017,Lin:SVRG}, we compare different algorithms based on the number of operator evaluations.

\begin{figure}[!htbp]
     \centering
     \begin{subfigure}[b]{0.48\textwidth}
         \centering
         \includegraphics[width=\textwidth]{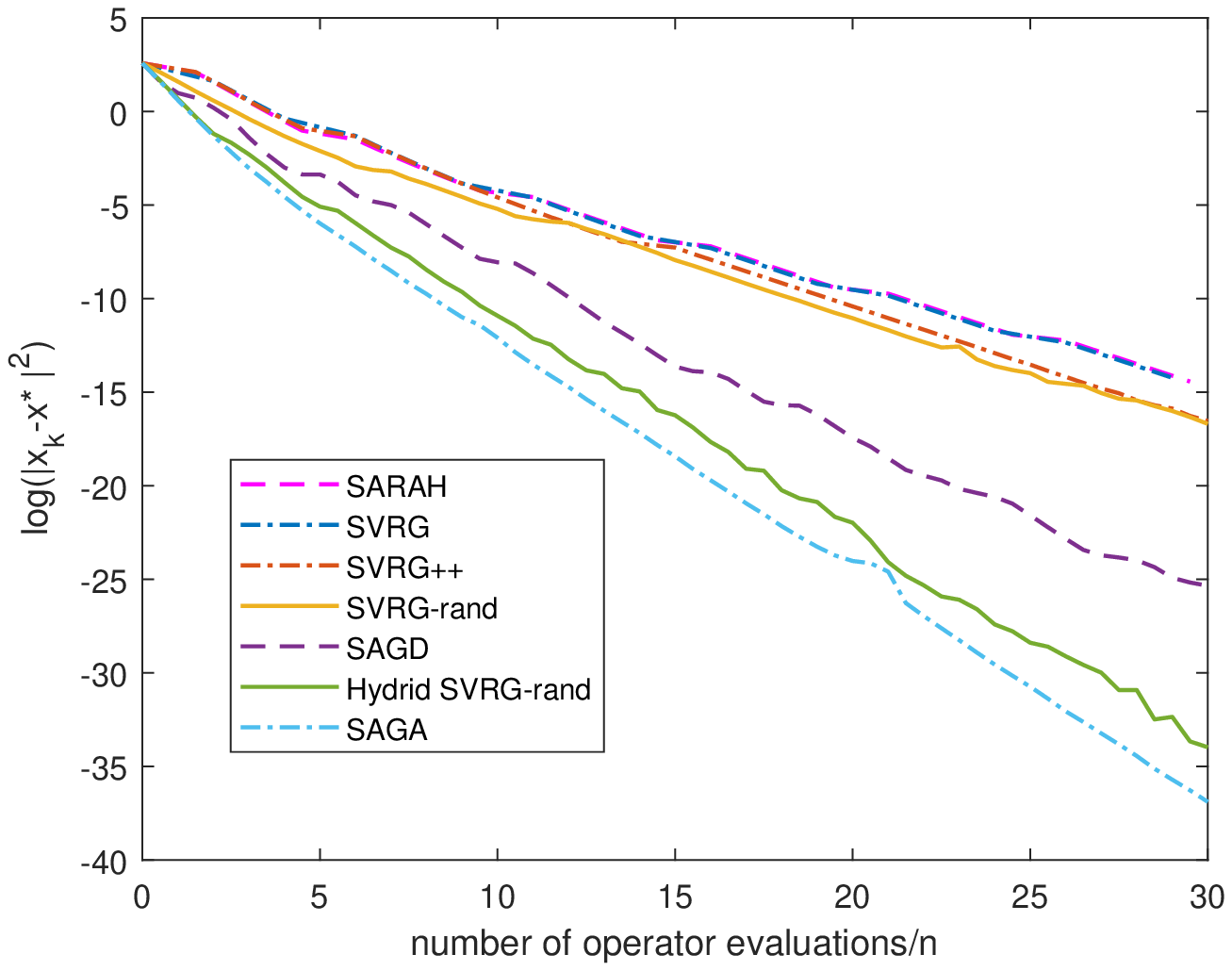}
         \caption{saddle-point problem} 
         \label{fig:saddle}
     \end{subfigure}
     \hfill
     \begin{subfigure}[b]{0.48\textwidth}
         \centering
         \includegraphics[width=\textwidth]{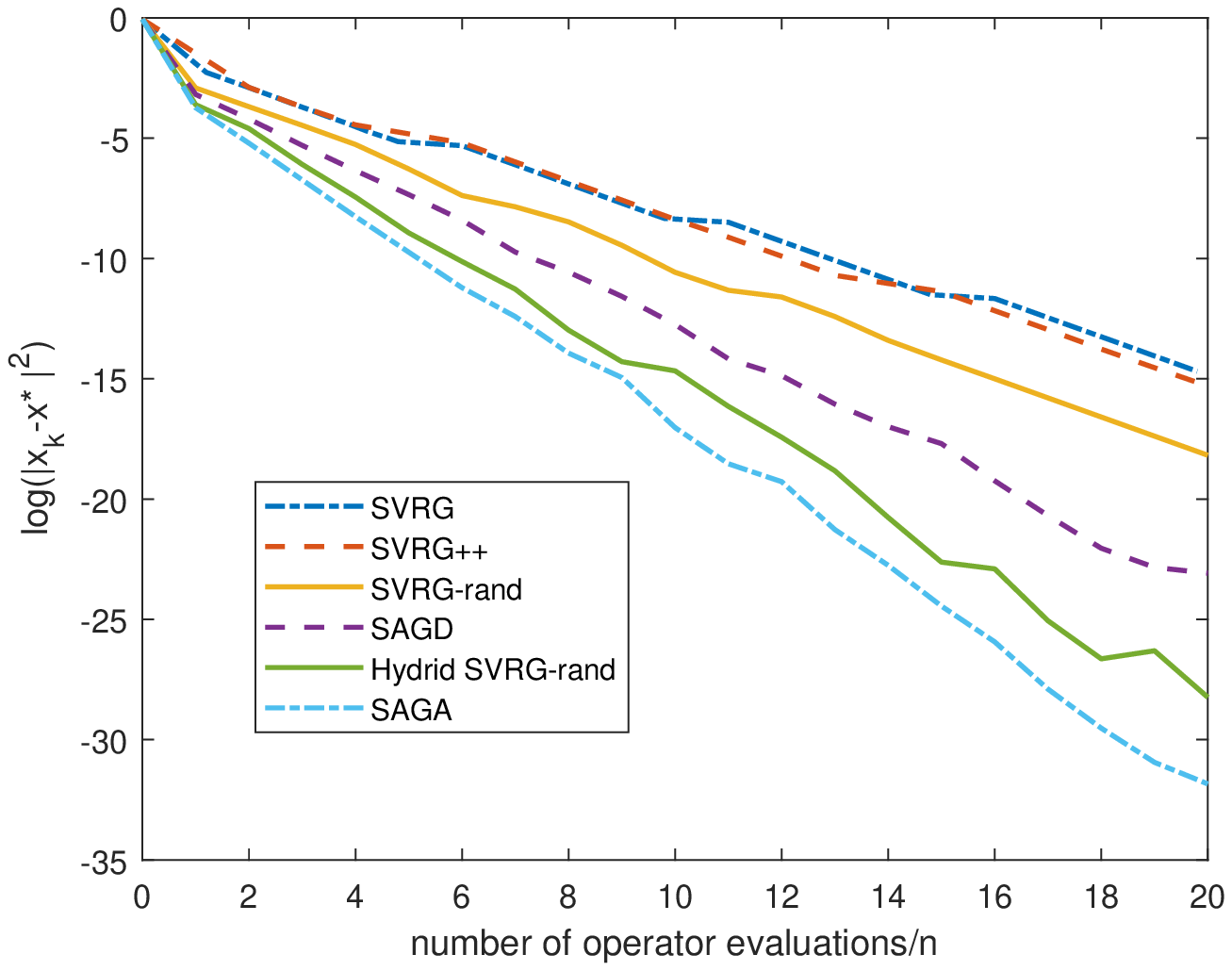}
         \caption{two-player game}
         \label{fig:game}
     \end{subfigure}
      \caption{Comparison of different algorithms on two problems}
\end{figure}

\paragraph*{Policy evaluation}

Consider the saddle-point problem from \citet{Policyeva} given by:
\begin{align}\label{saddle_mspbe}
\min_{\theta \in \mathbb{R}^d}\max_{\omega \in \mathbb{R}^d}  ~~\frac{1}{n} \sum_{i=1}^n \left( \omega^Tb_i-\omega^T A_i\theta- \frac{1}{2}\omega^{T}C_i\omega +\frac{\lambda}{2}\| \theta \|^2  \right) ,
\end{align}
where $\lambda>0$, $\theta, \, \omega, \, b_i \in \mathbb{R}^d$, and $A_i, \, C_i \in \mathbb{R}^{d\times d }$.
Problem (\ref{saddle_mspbe}) is closely related to the policy evaluation problem in dynamic programming \citep{policy_eva_survey,Policyeva}. 
In this experiment, we test the algorithms on the Boyan chain benchmark \citep{Boyan2002}.
Figure \ref{fig:saddle} shows the behaviour of our algorithms for solving Problem \eqref{saddle_mspbe}. 
The  horizontal axis is the number of operator evaluations divided by $n$ and the vertical axis is the logarithm of the distance to the optimum. 
From this figure, we can see that SAGA has the best performance, followed by hybrid SAGA and SVRG-rand. 
However, to achieve this superior performance, SAGA inherently requires more storage than either hybrid SAGA or SVRG-rand. 
We also observe that SAGD performs worse than either SAGA or SAGA+SVRG-rand, even though it has the same storage cost as SAGA. 
However, SAGD does outperform the four algorithms that do not incur storage cost (SARAH, SVRG, SVRG++, SVRG-rand).
Among these four algorithms that do not require storage cost, our proposed SVRG-rand performs best in  both earlier and later iterations.
SVRG++ has similar performance to SVRG in early iterations and better performance than SVRG in later iterations.
SARAH performs almost the same as SVRG in both earlier and later iterations.

\paragraph*{Strongly monotone games}

Next consider a general finite sum strongly monotone two player game (see, e.g., \cite{tatarenko2018learning}). Let $A_1,\,A_2 \subset \mathbb{R}^d$ be closed and convex action sets for each player, and let $A = A_1 \times A_2$ be the set of joint actions. 
Following \cite{game_two_player}, we consider $a_1, a_2\in \mathbb{R}_{+}$ satisfying
$a_1+a_2\leq 1.$

Let $f_j : A \rightarrow \mathbb R$ be the cost function for each player $j = 1, 2$. We suppose that $f_j(a_j, a_{-j})$ is continuously differentiable in $a=(a_j, a_{-j})$ and $\mu$-strongly convex in $a_j$ with an $L$-Lipschitz gradient, for $j = 1, 2$. To motivate the finite sum structure of this problem, we recall that for fixed actions $a_1$ and $a_2$, the cost functions are often an expectation over a set of possible scenarios \citep{game_theory}.
In this study, we take the specific cost functions
$$
f_1(a_1, a_2) = \frac{1}{n}\sum_{i = 1}^n \left(a_1^T \left( b_{i, 1} + A_{i, 1} a_2 \right) + \frac{1}{2} a_1^T C_{i, 1} a_1 \right),
$$
$$
f_2(a_1, a_2) = \frac{1}{n}\sum_{i = 1}^n \left(a_2^T \left( b_{i, 2} + A_{i, 2} a_1 \right) + \frac{1}{2} a_2^T C_{i, 2} a_2 \right),
$$
which correspond to using  sample average approximation (SAA) based on historical data $\{ b_{i,1},A_{i,1}, C_{i,1}, b_{i,2},A_{i,2}, C_{i,2} \}_{i\geq 1}$, and we assume the costs depend linearly on the complementary strategy.

Player $j$'s best response problem given $a_{-j}$ is $\min_{a_j \in A_j} f_j(a_j,\,a_{-j})$. The {\em game mapping} is defined as
$M(a) \triangleq (\nabla_{a_1} f_1(a_1,\,a_2),\, \nabla_{a_2} f_2(a_1,\,a_2))$,
where for our specific game we have
\[
M(a) = \frac{1}{n} \sum_{i=1}^n \left(M_i\cdot (a_1, a_2)^T + b_i\right)
\]
where
\[
M_i \triangleq \left[\begin{array}{cc}
C_{i,1} & A_{i,1}\\
A_{i,2} & C_{i,2}
\end{array}\right],\,\forall i \in [n],
\]
and $b_i \triangleq (b_{i,1}, b_{i,2})^T$ for all $i \in [n]$. The corresponding VI formulation for a Nash equilibrium is to find $a^* \in A$ such that $\langle M(a^*), a - a^* \rangle \geq 0$ for all $a \in A$. This VI is equivalent to the monotone inclusion problem $0 \in M(a) + N_A(a)$.

Figure \ref{fig:game} shows the behaviour of our algorithms for computing an equilibrium of this game. 
Since $A$ is nonzero in this example, we do not test SARAH.
The behaviour of our algorithms for solving the two-player game matches their behavior for the saddle-point problem. 
However, the gap in performance between SVRG-rand and SVRG is now larger (and still favors SVRG-rand).

In summary, by evaluating the performance of different variance-reduced algorithms in two  different examples, we observe a  trade-off between storage cost and convergence rate. 
Algorithms with storage cost generally perform better than algorithms without storage cost, and this observation is consistent with the literature \citep{Balamurugan2016,SAG:2012}.
These experiments also help to show that the performance of our two newly proposed algorithms is satisfactory.
When no additional storage is allowed, we observe that the new algorithm SVRG-rand is faster than classical SVRG and SVRG++.
If additional storage is allowed, then SAGA+SVRG-rand performs better than SAGD, and SAGA converges slightly faster than SAGA+SVRG-rand but at greater storage cost.

\section{Conclusion and future work}\label{sec:conclusion}
We have presented a unifying framework for variance-reduced FB splitting for finding zeroes of the sum of two monotone operators, where one is the average of a large number of strongly maximal monotone and Lipschitz operators, and the other is a general maximal monotone operator.
Our framework covers many popular variance-reduced algorithms for solving finite sum minimization problems as well as their extension to monotone inclusion problems.

The basis of our technique is a Lyapunov-type argument which we use to establish linear convergence of a class of algorithms including GD, SVRG, SAGA, HSAG, and several others. This argument reveals that all of these variance-reduced algorithms can be understood in some sense as iteration of near contraction operators subject to a disturbance, and that the effect of the disturbance decays to zero over time when appropriate conditions are met. This argument also extends to the design of new algorithms which similarly enjoy a linear convergence rate in expectation as well as enjoy favorable numerical properties. 
 We further show that our framework can be accelerated by Catalyst and asynchronous implementation.

This work is only the beginning of the analysis of variance-reduced algorithms for monotone inclusion problems.
Unbiased operator estimates in the inner forward step played a fundamental role in our analysis.  In future work, we will consider SAG and other biased variance-reduced algorithms, as well as other methods for acceleration.


\acks{We would like to acknowledge support for this project
from the National University of Singapore Young Investigator Award ``Practical considerations for large-scale competitive decision making''.}
We also gratefully acknowledge the helpful comments and suggestions from the associate editor and the two anonymous referees.

\vskip 0.2in
\bibliography{acl}

\begin{thebibliography}{40}
\providecommand{\natexlab}[1]{#1}
\providecommand{\url}[1]{\texttt{#1}}
\expandafter\ifx\csname urlstyle\endcsname\relax
  \providecommand{\doi}[1]{doi: #1}\else
  \providecommand{\doi}{doi: \begingroup \urlstyle{rm}\Url}\fi

\bibitem[Agarwal and Bottou(2015)]{lowerbound_of_mini}
Alekh Agarwal and L\'{e}on Bottou.
\newblock A lower bound for the optimization of finite sums.
\newblock In \emph{Proceedings of the 32nd International Conference on
  International Conference on Machine Learning - Volume 37}, ICML'15, page
  78–86. JMLR.org, 2015.

\bibitem[Allen-Zhu and Yuan(2016)]{svrg++}
Zeyuan Allen-Zhu and Yang Yuan.
\newblock Improved svrg for non-strongly-convex or sum-of-non-convex
  objectives.
\newblock In \emph{Proceedings of the 33rd International Conference on
  International Conference on Machine Learning - Volume 48}, ICML'16, pages
  1080--1089. JMLR.org, 2016.

\bibitem[Balamurugan and Bach(2016)]{Balamurugan2016}
P.~Balamurugan and Francis Bach.
\newblock Stochastic variance reduction methods for saddle-point problems.
\newblock In \emph{Proceedings of the 30th International Conference on Neural
  Information Processing Systems}, NIPS'16, pages 1416--1424, USA, 2016. Curran
  Associates Inc.
\newblock ISBN 978-1-5108-3881-9.

\bibitem[Bauschke and Combettes(2011)]{monotoneiclusion3}
Heinz~H. Bauschke and Patrick~L. Combettes.
\newblock \emph{Convex Analysis and Monotone Operator Theory in Hilbert
  Spaces}.
\newblock Springer Publishing Company, Incorporated, 1st edition, 2011.
\newblock ISBN 1441994661, 9781441994660.

\bibitem[{Bibi} et~al.(2018){Bibi}, {Sailanbayev}, {Ghanem}, {Mansel Gower},
  and {Richt{\'a}rik}]{SAGD}
Adel {Bibi}, Alibek {Sailanbayev}, Bernard {Ghanem}, Robert {Mansel Gower}, and
  Peter {Richt{\'a}rik}.
\newblock {Improving SAGA via a Probabilistic Interpolation with Gradient
  Descent}.
\newblock \emph{arXiv e-prints}, art. arXiv:1806.05633, June 2018.

\bibitem[Boyan(2002)]{Boyan2002}
Justin~A. Boyan.
\newblock Technical update: Least-squares temporal difference learning.
\newblock \emph{Machine Learning}, 49\penalty0 (2):\penalty0 233--246, Nov
  2002.
\newblock ISSN 1573-0565.

\bibitem[Boyd and Vandenberghe(2004)]{boyd}
Stephen Boyd and Lieven Vandenberghe.
\newblock \emph{Convex Optimization}.
\newblock {Cambridge University Press}, March 2004.
\newblock ISBN 0521833787.

\bibitem[Bubeck(2015)]{Ber_cv}
S{\'e}bastien Bubeck.
\newblock Convex optimization: Algorithms and complexity.
\newblock \emph{Foundations and Trends in Machine Learning}, 8\penalty0
  (3-4):\penalty0 231--357, November 2015.
\newblock ISSN 1935-8237.

\bibitem[Chen et~al.(2017)Chen, Lan, and Ouyang]{chen2017accelerated}
Yunmei Chen, Guanghui Lan, and Yuyuan Ouyang.
\newblock Accelerated schemes for a class of variational inequalities.
\newblock \emph{Mathematical Programming}, 165\penalty0 (1):\penalty0 113--149,
  2017.

\bibitem[{Cortes} and {Vapnik}(1995)]{Cortes1995svm}
C.~{Cortes} and V.~{Vapnik}.
\newblock Support-vector networks.
\newblock \emph{Machine Learning}, 20\penalty0 (3):\penalty0 273--297, Sep
  1995.
\newblock ISSN 1573-0565.

\bibitem[Dann et~al.(2014)Dann, Neumann, and Peters]{policy_eva_survey}
Christoph Dann, Gerhard Neumann, and Jan Peters.
\newblock Policy evaluation with temporal differences: A survey and comparison.
\newblock \emph{Journal of Machine Learning Research}, 15:\penalty0 809--883,
  2014.

\bibitem[{Defazio} et~al.(2014){Defazio}, {Bach}, and
  {Lacoste-Julien}]{DefazioSAGA}
A.~{Defazio}, F.~{Bach}, and S.~{Lacoste-Julien}.
\newblock Saga: A fast incremental gradient method with support for
  non-strongly convex composite objectives.
\newblock In Z.~Ghahramani, M.~Welling, C.~Cortes, N.~D. Lawrence, and K.~Q.
  Weinberger, editors, \emph{Advances in Neural Information Processing Systems
  27}, pages 1646--1654. Curran Associates, Inc., 2014.

\bibitem[Du et~al.(2017)Du, Chen, Li, Xiao, and Zhou]{Policyeva}
Simon~S. Du, Jianshu Chen, Lihong Li, Lin Xiao, and Dengyong Zhou.
\newblock Stochastic variance reduction methods for policy evaluation.
\newblock In Doina Precup and Yee~Whye Teh, editors, \emph{Proceedings of the
  34th International Conference on Machine Learning}, volume~70 of
  \emph{Proceedings of Machine Learning Research}, pages 1049--1058,
  International Convention Centre, Sydney, Australia, 06--11 Aug 2017. PMLR.

\bibitem[{Dumais} et~al.(1998){Dumais}, {Platt}, {Heckerman}, and
  {Sahami}]{Dumais1998text}
S.~{Dumais}, J.~{Platt}, D.~{Heckerman}, and M.~{Sahami}.
\newblock Inductive learning algorithms and representations for text
  categorization.
\newblock In \emph{Proceedings of the Seventh International Conference on
  Information and Knowledge Management}, CIKM '98, pages 148--155, New York,
  NY, USA, 1998. ACM.
\newblock ISBN 1-58113-061-9.

\bibitem[Harker and Pang(1990)]{v-i}
Patrick~T. Harker and Jong-Shi Pang.
\newblock Finite-dimensional variational inequality and nonlinear
  complementarity problems: A survey of theory, algorithms and applications.
\newblock \emph{Mathematical Programming}, 48\penalty0 (1):\penalty0 161--220,
  Mar 1990.
\newblock ISSN 1436-4646.
\newblock \doi{10.1007/BF01582255}.
\newblock URL \url{https://doi.org/10.1007/BF01582255}.

\bibitem[Hsieh et~al.(2015)Hsieh, Yu, and Dhillon]{asdca}
Cho-Jui Hsieh, Hsiang-Fu Yu, and Inderjit~S. Dhillon.
\newblock Passcode: Parallel asynchronous stochastic dual co-ordinate descent.
\newblock In \emph{Proceedings of the 32Nd International Conference on
  International Conference on Machine Learning - Volume 37}, ICML'15, pages
  2370--2379. JMLR.org, 2015.

\bibitem[{Johnson} and {Zhang}(2013)]{JohnsonSVRG}
R.~{Johnson} and T.~{Zhang}.
\newblock Accelerating stochastic gradient descent using predictive variance
  reduction.
\newblock In C.~J.~C. Burges, L.~Bottou, M.~Welling, Z.~Ghahramani, and K.~Q.
  Weinberger, editors, \emph{Advances in Neural Information Processing Systems
  26}, pages 315--323. Curran Associates, Inc., 2013.

\bibitem[Kattsoff(1945)]{game_theory}
Louis~O. Kattsoff.
\newblock Theory of games and economic behavior. by john von neumann and oskar
  morgenstern. princeton, n. j.: Princeton university press, 1944. 625 pp.
\newblock \emph{Social Forces}, 24\penalty0 (2):\penalty0 245--246, 12 1945.
\newblock ISSN 0037-7732.
\newblock \doi{10.2307/2572550}.
\newblock URL \url{https://doi.org/10.2307/2572550}.

\bibitem[Konečný and Richtárik(2017)]{S2GD}
Jakub Konečný and Peter Richtárik.
\newblock Semi-stochastic gradient descent methods.
\newblock \emph{Frontiers in Applied Mathematics and Statistics}, 3:\penalty0
  9, 2017.
\newblock ISSN 2297-4687.

\bibitem[Leblond et~al.(2017)Leblond, Pedregosa, and Lacoste-Julien]{ASAGA}
Rémi Leblond, Fabian Pedregosa, and Simon Lacoste-Julien.
\newblock {ASAGA: Asynchronous Parallel SAGA}.
\newblock In Aarti Singh and Jerry Zhu, editors, \emph{Proceedings of the 20th
  International Conference on Artificial Intelligence and Statistics},
  volume~54 of \emph{Proceedings of Machine Learning Research}, pages 46--54,
  Fort Lauderdale, FL, USA, 20--22 Apr 2017. PMLR.

\bibitem[Lin et~al.(2015)Lin, Mairal, and Harchaoui]{Catalyst}
Hongzhou Lin, Julien Mairal, and Zaid Harchaoui.
\newblock A universal catalyst for first-order optimization.
\newblock In \emph{Proceedings of the 28th International Conference on Neural
  Information Processing Systems - Volume 2}, NIPS'15, pages 3384--3392,
  Cambridge, MA, USA, 2015. MIT Press.
\newblock URL \url{http://dl.acm.org/citation.cfm?id=2969442.2969617}.

\bibitem[Lin et~al.(2018)Lin, Mairal, and Harchaoui]{catalyst_2}
Hongzhou Lin, Julien Mairal, and Zaid Harchaoui.
\newblock {Catalyst Acceleration for First-order Convex Optimization: from
  Theory to Practice}.
\newblock \emph{{Journal of Machine Learning Research}}, 18\penalty0
  (1):\penalty0 7854--7907, April 2018.
\newblock URL \url{https://hal.inria.fr/hal-01664934}.
\newblock http://jmlr.org/papers/volume18/17-748/17-748.pdf.

\bibitem[Liu and Nocedal(1989)]{Liu1989}
Dong~C. Liu and Jorge Nocedal.
\newblock On the limited memory bfgs method for large scale optimization.
\newblock \emph{Mathematical Programming}, 45\penalty0 (1):\penalty0 503--528,
  Aug 1989.
\newblock ISSN 1436-4646.

\bibitem[Liu et~al.(2015)Liu, Wright, R{\'e}, Bittorf, and Sridhar]{ascda}
Ji~Liu, Stephen~J. Wright, Christopher R{\'e}, Victor Bittorf, and Srikrishna
  Sridhar.
\newblock An asynchronous parallel stochastic coordinate descent algorithm.
\newblock \emph{Journal of Machine Learning Research}, 16\penalty0
  (1):\penalty0 285--322, January 2015.
\newblock ISSN 1532-4435.

\bibitem[Mania et~al.(2017)Mania, Pan, Papailiopoulos, Recht, Ramchandran, and
  Jordan]{mania2015}
H.~Mania, X.~Pan, D.~Papailiopoulos, B.~Recht, K.~Ramchandran, and M.~Jordan.
\newblock Perturbed iterate analysis for asynchronous stochastic optimization.
\newblock \emph{SIAM Journal on Optimization}, 27\penalty0 (4):\penalty0
  2202--2229, 2017.

\bibitem[{Min} et~al.(2017){Min}, {Zhao}, {Long}, {Wu}, {Li}, and
  {Yin}]{svrg_adapt}
E.~{Min}, Y.~{Zhao}, J.~{Long}, C.~{Wu}, K.~{Li}, and J.~{Yin}.
\newblock Svrg with adaptive epoch size.
\newblock In \emph{2017 International Joint Conference on Neural Networks
  (IJCNN)}, pages 2935--2942, May 2017.

\bibitem[Nesterov(2014)]{Nesterov:2014:ILC:2670022}
Yurii Nesterov.
\newblock \emph{Introductory Lectures on Convex Optimization: A Basic Course}.
\newblock Springer Publishing Company, Incorporated, 1 edition, 2014.
\newblock ISBN 1461346916, 9781461346913.

\bibitem[{Nguyen} et~al.(2017){Nguyen}, {Liu}, {Scheinberg}, and
  {Tak{\'a}{\v{c}}}]{SARAH}
Lam~M. {Nguyen}, Jie {Liu}, Katya {Scheinberg}, and Martin {Tak{\'a}{\v{c}}}.
\newblock {SARAH: A Novel Method for Machine Learning Problems Using Stochastic
  Recursive Gradient}.
\newblock \emph{arXiv e-prints}, art. arXiv:1703.00102, February 2017.

\bibitem[{Nicolas} et~al.(2012){Nicolas}, {Schmidt}, and {Francis}]{SAG:2012}
L.~{Nicolas}, M.~{Schmidt}, and R~{Francis}.
\newblock A stochastic gradient method with an exponential convergence rate for
  finite training sets.
\newblock In F.~Pereira, C.~J.~C. Burges, L.~Bottou, and K.~Q. Weinberger,
  editors, \emph{Advances in Neural Information Processing Systems 25}, pages
  2663--2671. Curran Associates, Inc., 2012.

\bibitem[Niu et~al.(2011)Niu, Recht, Re, and Wright]{asgd}
Feng Niu, Benjamin Recht, Christopher Re, and Stephen~J. Wright.
\newblock Hogwild!: A lock-free approach to parallelizing stochastic gradient
  descent.
\newblock In \emph{Proceedings of the 24th International Conference on Neural
  Information Processing Systems}, NIPS'11, pages 693--701, USA, 2011. Curran
  Associates Inc.
\newblock ISBN 978-1-61839-599-3.

\bibitem[Nocedal(1980)]{1980}
Jorge Nocedal.
\newblock Updating quasi-newton matrices with limited storage.
\newblock \emph{Mathematics of Computation}, 35\penalty0 (151):\penalty0
  773--782, 1980.
\newblock ISSN 00255718, 10886842.

\bibitem[Passty(1979)]{FBspit}
Gregory~B Passty.
\newblock Ergodic convergence to a zero of the sum of monotone operators in
  hilbert space.
\newblock \emph{Journal of Mathematical Analysis and Applications}, 72\penalty0
  (2):\penalty0 383 -- 390, 1979.
\newblock ISSN 0022-247X.

\bibitem[{Reddi} et~al.(2015){Reddi}, {Hefny}, {Sra}, {Poczos}, and
  {Smola}]{reddi2015variance}
J.~{Reddi}, A.~{Hefny}, S.~{Sra}, B.~{Poczos}, and A.~{Smola}.
\newblock On variance reduction in stochastic gradient descent and its
  asynchronous variants.
\newblock In C.~Cortes, N.~D. Lawrence, D.~D. Lee, M.~Sugiyama, and R.~Garnett,
  editors, \emph{Advances in Neural Information Processing Systems 28}, pages
  2647--2655. Curran Associates, Inc., 2015.

\bibitem[{Robbins} and {Monro}(1951)]{robbins1951}
H.~{Robbins} and S.~{Monro}.
\newblock A stochastic approximation method.
\newblock \emph{The Annals of Mathematical Statistics}, 22\penalty0
  (3):\penalty0 400--407, 09 1951.

\bibitem[Rosasco et~al.(2016)Rosasco, Villa, and Vũ]{FBSGD}
Lorenzo Rosasco, Silvia Villa, and Băng~Công Vũ.
\newblock A stochastic inertial forward--backward splitting algorithm for
  multivariate monotone inclusions.
\newblock \emph{Optimization}, 65\penalty0 (6):\penalty0 1293--1314, 2016.

\bibitem[{Ryu} and {Boyd}(2016)]{survey}
E.~{Ryu} and S.~{Boyd}.
\newblock A primer on monotone operator methods survey.
\newblock \emph{Applied and computational mathematics}, 15:\penalty0 3--43, 01
  2016.

\bibitem[{Schmidt} et~al.(2017){Schmidt}, {Le Roux}, and {Bach}]{Schmidt2017}
M.~{Schmidt}, N.~{Le Roux}, and F.~{Bach}.
\newblock Minimizing finite sums with the stochastic average gradient.
\newblock \emph{Mathematical Programming}, 162\penalty0 (1):\penalty0 83--112,
  Mar 2017.
\newblock ISSN 1436-4646.

\bibitem[{Tatarenko} and {Kamgarpour}(2019)]{game_two_player}
T.~{Tatarenko} and M.~{Kamgarpour}.
\newblock Learning generalized nash equilibria in a class of convex games.
\newblock \emph{IEEE Transactions on Automatic Control}, 64\penalty0
  (4):\penalty0 1426--1439, 2019.
\newblock \doi{10.1109/TAC.2018.2841319}.

\bibitem[Tatarenko and Kamgarpour(2018)]{tatarenko2018learning}
Tatiana Tatarenko and Maryam Kamgarpour.
\newblock Learning generalized nash equilibria in a class of convex games.
\newblock \emph{IEEE Transactions on Automatic Control}, 64\penalty0
  (4):\penalty0 1426--1439, 2018.

\bibitem[{Xiao} and {Zhang}(2014)]{Lin:SVRG}
L.~{Xiao} and T.~{Zhang}.
\newblock A proximal stochastic gradient method with progressive variance
  reduction.
\newblock \emph{SIAM Journal on Optimization}, 24\penalty0 (4):\penalty0
  2057--2075, 2014.

\end{thebibliography}

\appendix

\section{Proofs for Section \ref{sec:preli} (forward-backward splitting)}

\subsection{Proof of Theorem \ref{FB-thm} (linear convergence rate)}

We have $x^*= (I + \gamma A)^{-1}(I-\gamma B)(x^*)$ for all $\gamma>0$ since $x^*$ is the unique solution of \eqref{eq:inclusion}. Then, we have
\begin{align*}
    \|x_{k+1} - x^*\|^2 = \, & \|(I + \gamma A)^{-1}(I - \gamma B)(x_k) - (I + \gamma A)^{-1}(I - \gamma B)(x^*)\|^2 \\
    \leq \, & \|(I - \gamma B)(x_k) - (I - \gamma B)(x^*)\|^2 \\
    = \, & \|x_k - x^*\|^2 - 2\,\gamma \langle B(x_k) - B(x^*),\, x_k - x^* \rangle + \gamma^2 \|B(x_k) - B(x^*) \|^2 \\
    \leq \, & (1 - 2 \gamma \mu +\gamma^2 L^2) \|x_k -x^*\|^2,
\end{align*}
where the first inequality uses non-expansiveness of the resolvent $(I + \gamma\,A)^{-1}$, and the second inequality uses $\mu$-strong monotonicity and $L$-Lipschitz continuity of $B$.

\section{Proofs for Section \ref{sec:forward-backward} (randomized forward-backward splitting)}

\subsection{Proof of Lemma \ref{lemma1} (one-step expected error)}

Using the shorthand $\mathcal{G}_k = \mathcal{G}\left(x_{k},\,\phi^k,I_{k}\right)$, we have
\begin{align*}
\|x_{k+1}-x^{*}\|^{2} 
&=\|(I+\gamma_k A)^{-1}(x_{k}-\gamma_k \mathcal{G}_k)
-(I+\gamma_k A)^{-1}(x^{*}-\gamma_k B(x^{*}))\|^{2}\\
&\leq \left\| \left(x_{k}-\gamma_k \mathcal{G}_k\right)-\left(x^{*}-\gamma_k B(x^{*})\right)  \right\|^2 \\
&=\|(x_{k}-x^{*})-\gamma_k \left(B(x_{k})-B(x^{*})\right) - \gamma_k \left(\mathcal{G}_k-B(x_{k})\right)\|^{2}\\
&=\|x_{k}-x^{*}\|^{2} +\gamma_k^{2}\| B(x_{k})-B(x^{*})\|^{2}+\gamma_k^{2}\|\mathcal{G}_k-B(x_{k})\|^{2} \\
&~~~ -2\gamma_k (B(x_{k})-B(x^{*}))^{T}(x_{k}-x^{*}) -2\gamma_k (x_{k}-x^{*})^{T}\left(\mathcal{G}_k-B(x_{k})\right)\\
&~~~+ 2\gamma_k^2(B(x_{k})-B(x^{*}))^{T}\left(\mathcal{G}_k-B(x_{k})\right),
\end{align*} 
where the last equality is by expansion of the squared norm. Using unbiasedness of $\mathcal{G}_k$, we see that
\begin{align*}
\mathbb{E}\left[\|x_{k+1}-x^{*}\|^{2}|\mathcal{F}_{k}\right]\leq &\|x_{k}-x^{*}\|^{2}-2\gamma_k (B(x_{k})-B(x^{*}))^{T}(x_{k}-x^{*})\\
& +\gamma_k^{2}\| B(x_{k})-B(x^{*})\|^{2}+\gamma_k^{2}\mathbb{E}\left[\|\mathcal{G}_k -B(x_{k})\|^{2}|\mathcal{F}_{k}\right].
\end{align*}
By strong monotonicity and Lipschitz continuity of $B$, we can bound the RHS of the above display with
\begin{equation*}
\mathbb{E}\left[\|x_{k+1}-x^{*}\|^{2}|\mathcal{F}_{k}\right]\leq (1-2\gamma_k \mu +\gamma_k^{2}L^{2})\|x_{k}-x^{*}\|^{2}+\gamma_k^{2}\mathbb{E}\left[\|\mathcal{G}_k-B(x_{k}) \|^{2}|\mathcal{F}_{k}\right] .
\end{equation*}
Taking expectations of both sides gives the desired conclusion.

\subsection{Proof of Lemma \ref{lemma2} (bound on conditional variance)}

Direct calculation shows that
\begin{align*}
\mathbb{E}\left[\|\mathcal{G}\left(x_{k},\,\phi^k,I_{k}\right)-B(x_{k})\|^{2}|\mathcal{F}_{k}\right]
&=\mathbb{E}\left[\left\|B_{I_{k}}\left(x_{k}\right)-\phi_{I_{k}}^{k}-\mathbb{E}\left[B_{I_{k}}-\phi_{I_{k}}^{k}|\mathcal{F}_{k} \right]\right\|^2|\mathcal{F}_{k}\right]\\
&\leq \mathbb{E}\left[\left\| B_{I_{k}}\left(x_{k}\right)-\phi_{I_{k}}^{k}\right\|^{2}|\mathcal{F}_{k}\right],
\end{align*}
where the inequality follows because the conditional variance of a random variable is not larger than its second moment. 
By adding and subtracting $B_{I_k}(x^*)$, we can further bound the last term of the above display with
\begin{align*}
\mathbb{E}\left[\| B_{I_{k}}\left(x_{k}\right)-\phi_{I_{k}}^{k}\|^{2}|\mathcal{F}_{k}\right]                                    &=\mathbb{E}\left[\|\left(\left(B_{I_{k}}(x_{k})-B_{I_{k}}(x^{*}\right)\right)-\left(\phi_{I_{k}}^{k}-B_{I_{k}}\left(x^{*}\right)\right)\|^{2}|\mathcal{F}_{k}\right]\\
&\leq 2\, \mathbb{E}\left[\|(B_{I_{k}}(x_{k})-B_{I_{k}}(x^{*})\|^{2}+\|\phi_{I_{k}}^{k}-B_{I_{k}}(x^{*})\|^{2}|\mathcal{F}_{k}\right]\\
&\leq 2\left(L^2\|x_{k}-x^{*}\|^{2}+\frac{1}{n}\sum_{i=1}^{n}\|\phi_{i}^{k}-B_{i}(x^{*})\|^{2}\right),
\end{align*}
by the triangle inequality. 
The first term on the RHS of the last inequality is due to $L$-Lipschitz continuity of each $B_{i}.$ 
The second term on the RHS of the last inequality is the conditional expectation of $\|\phi_{I_{k}}^{k}-B_{I_{k}}(x^{*})\|^{2} $ with respect to $\mathcal{F}_{k}$.
Take full expectation of both sides of the above display to conclude.

\subsection{Proof of Theorem \ref{main theorem} (linear convergence rate)}

We first show that $\theta<1$. 
For the choices of $\gamma$ and $\rho$ given in the theorem, $\theta$ is the maximum of $\ell_1$ and $\ell_2$ where
$$
\ell_1 \triangleq  1-2\gamma\mu+3\gamma^2L^2+c_2\gamma^{1.5},~~ \ell_2 \triangleq 2\gamma^{0.5}+c_1.
$$
Showing that $\theta<1$ is achieved by showing that both $\ell_1$ and $\ell_2$ are less than one. 
By decreasing the denominators of both terms inside the bracket of \eqref{eq:gamma ineq}, we have
\begin{equation}\label{eq:gamma refined ineq}
\gamma<\min\left\lbrace \left(\frac{2\mu}{\frac{3(1-c_1)L^2}{2}+c_2}\right)^2, ~ \frac{\left(1-c_1\right)^2}{4} \right\rbrace.
\end{equation}
It is then readily seen that $\ell_2<1$ because $\gamma<\left(1-c_1\right)^2/4$. 
To show $\ell_1<1$, note that 
$$
3\gamma L^2+c_2\gamma^{0.5}=\gamma^{0.5}\left(3\gamma^{0.5}L^2+c_2\right)
<\frac{2\mu}{\frac{3(1-c_1)L^2}{2}+c_2}\left(\frac{3(1-c_1)L^2}{2}+c_2\right) = 2 \mu,
$$
where we use the fact that $\gamma$ is smaller than the first term in the bracket of \eqref{eq:gamma refined ineq} to get the first term in the last inequality, 
and the fact that $\gamma$ is smaller than the second term in the bracket of \eqref{eq:gamma refined ineq} to get the second term in the last inequality.
By expressing $\ell_1$ as $1-\gamma(2\mu - 3\gamma L^2 - c_2\gamma^{0.5})$, it is readily see that $\ell_1<1$ and so $\theta<1$.

Next, we establish inequality \eqref{eq:L_rho contraction}.
Based on \hyperref[coro1]{Corollary~\ref*{coro1}} and \hyperref[assum2]{Assumption~\ref*{assum2}}, we have
\begin{align}
\mathbb{E}L_{\rho}\left(\tilde{x}_{k},\tilde{\phi}^{k}\right)
&= \mathbb{E}\|x_{S_k}-x^{*}\|^{2}+\rho \mathbb{E} G\left(\phi^{S_k}\right) \nonumber\\
& \leq \left(1-2\gamma \mu+3\gamma^{2}L^2+c_2\rho\right) \mathbb{E} \|x_{S_{k}-1}-x^{*}\|^{2} 
\nonumber\\
&~~~+\left(\frac{2\gamma^2}{\rho}+c_1\right)\rho  \mathbb{E} G\left(\phi^{S_k-1}\right) 
+2\gamma^2 \mathbb{E} H\left(\phi^{S_k-1}\right)\nonumber \\
&\leq \max \left\lbrace 1-2\gamma\mu+3\gamma^2L^2+c_2\rho,~\frac{2\gamma^2}{\rho}+c_1  \right\rbrace 
\mathbb{E} L_\rho\left(x_{S_k-1},\phi^{S_k-1}\right)\nonumber \\
&~~~+ 2\gamma^2 \mathbb{E} H\left(\phi^{S_k-1}\right) .\label{eq: contract of L}
\end{align}
Using the definition of $\theta$, the above display is equivalent to $\mathbb{E}L_{\rho}\left(\tilde{x}_{k},\tilde{\phi}^{k}\right)\leq
\theta \mathbb{E}L_\rho\left(x_{S_k-1},\phi^{S_k-1}\right) +2\gamma^2 \mathbb{E} H\left(\phi^{S_k-1}\right)$. By recursively applying this inequality, we obtain
\begin{align}
\mathbb{E}L_{\rho}\left(\tilde{x}_{k},\tilde{\phi}^{k}\right)
&\leq  \theta^{m_k}  \mathbb{E} L_{\rho}\left(x_{S_{k-1}},\phi^{S_{k-1}}\right) +2\gamma^2 \mathbb{E} \sum_{i=0}^{m_{k}-1}\theta^i H\left(\phi^{S_k-1-i}\right) \nonumber\\
&\leq \left(\theta^{m_k}+2\gamma^2c_3 \sum_{i=0}^{m_k-1}\theta^i\right) \mathbb{E} L_{\rho}\left(x_{S_{k-1}},\phi^{S_{k-1}}\right),\label{eq: exact contract of L}
\end{align}
where the last inequality is due to Assumption \ref{assum2}.2. 
Because $\theta<1$ and $m_k\leq\bar{m}$ for all $k \geq 0$, the above display implies
\begin{align*}
\mathbb{E} L_\rho\left(x_{S_k},\phi^{S_k}\right) 
\leq \left(\theta+2\gamma^2\bar{m}c_3\right) \mathbb{E} L_\rho\left(x_{S_{k-1}},\phi^{S_{k-1}}\right)
=\lambda \mathbb{E}L_\rho \left(x_{S_{k-1}},\phi^{S_{k-1}}\right).
\end{align*}

We complete the proof by showing that $\lambda<1.$ 
Using $\theta$ defined in (\ref{theta}) and $\lambda$ defined in (\ref{-lambda}), we may rewrite $\lambda$ as
\begin{align*}
 \lambda=\max \left\lbrace1-2\gamma\mu+3\gamma^2L^2+c_2\rho+2\gamma^2\bar{m}c_3 , \frac{2\gamma^2}{\rho}+c_1+2\gamma^2\bar{m}c_3\right\rbrace.
\end{align*}
With $\rho=\gamma^{1.5}$, $\lambda$ is the maximum of $\tilde{\ell}_1$ and $\tilde{\ell}_2$ where 
$$
\tilde{\ell}_1 \triangleq 1-2\gamma\mu+3\gamma^2L^2+c_2\gamma^{1.5} +2\gamma^2\bar{m}c_3,~~~\tilde{\ell}_2 \triangleq 2\gamma^{0.5}+2\gamma^2\bar{m}c_3+c_1.
$$
Now it suffices to show that both $\tilde{\ell}_1$ and $\tilde{\ell}_2$ are less than one. 
Starting with $\tilde{\ell}_2$, we have
\begin{align}\label{eq:B_1}
2\gamma^{0.5}+2\bar{m}\gamma^2c_3=\gamma^{0.5}(2+2\bar{m}\gamma^{1.5}c_3)<\frac{1-c_1}{2+2\bar{m}(\frac{1-c_1}{2})^3c_3}\left(2+2\bar{m}\gamma^{1.5}c_3\right),
\end{align}
where the inequality follows because $\gamma$ is less than the second term in the bracket of (\ref{eq:gamma ineq}). 
Since $\gamma\leq (1-c_1)^2/4$ by (\ref{eq:gamma refined ineq}), 
the RHS of (\ref{eq:B_1}) is less than $1-c_1$. 
From (\ref{eq:B_1}) we have that $2\gamma^{0.5}+2\bar{m}\gamma^2c_3<1-c_1$, 
and thus $\tilde{\ell}_2<1.$ 
To show $\tilde{\ell}_1<1$, note that
\begin{align*}
3\gamma L^2+c_2\gamma^{0.5}+2\gamma \bar{m}c_3&=\gamma^{0.5}\left(3\gamma^{0.5}L^2+2\bar{m}c_3\gamma^{0.5}+c_2\right)\\
&< \frac{2\mu}{\frac{3(1-c_1)L^2}{2}+(1-c_1)c_3\bar{m}+c_2}\left(3\gamma^{0.5}L^2+2\bar{m}c_3\gamma^{0.5}+c_2\right),
\end{align*}
where the inequality is because $\gamma$ is less than the first term in the bracket (\ref{eq:gamma ineq}). 
Since $\gamma\leq (1-c_1)^2/4$, the above display is less than $2\mu$. 
Finally, rewrite $\tilde{\ell}_1$ as $1+\gamma ( 3\gamma L^2+c_2\gamma^{0.5}+2\gamma \bar{m}c_3 -2\mu )$ to see that $\tilde{\ell}_1<1$. 

Now we consider the distance to the optimal solution. We have that $\lambda\in [0,1)$, so recursively applying inequality (\ref{eq:L_rho contraction}) and using Lipschitz continuity of $B_i$ gives
\begin{align*}
\mathbb{E}L_{\rho}\left(\tilde{x}_k,\tilde{\phi}^k\right)&\leq \lambda^k \mathbb{E}L_{\rho}\left(\tilde{x}_0,\tilde{\phi}^0\right)\\
&=\lambda^k\left(\|x_0-x^*  \|^2+ \frac{\rho}{n}\sum_{i\in \mathcal{S}} \|B_i(x_0)-B_i(x^*)  \|^2  \right)\\
&\leq \lambda^k\left( 1+\rho L^2 \right)\| x_0-x^* \|^2.
\end{align*}
Finally, we use the definition of $L_{\rho}$ to conclude that $\mathbb{E}\|\tilde{x}_k-x^*\|^2\leq  \mathbb{E}L_{\rho}\left(\tilde{x}_k,\tilde{\phi}^k\right)\leq \lambda^k\left( 1+\rho L^2 \right)\| x_0-x^* \|^2$.

\section{Proofs for Section \ref{sec:Catalyst} (Catalyst)}

\subsection{Proof of Lemma \ref{lem: linear of cata} (one-step expected error)}

By definition of ${x}^*(\bar{x})$, we have $A({x}^*(\bar{x}))+B({x}^*(\bar{x}))+\sigma{x}^*(\bar{x})=\sigma \check{x}_k$ and thus ${x}^*(\bar{x})=\left(  A+B+\sigma I\right)^{-1}(\sigma \check{x}_k)$ for all $\sigma \geq 0$. 
Similarly, we have $x^*=\left(  A+B+\sigma I\right)^{-1}(\sigma x^*)$ for all $\sigma \geq 0$.
Using these two equalities we may bound $\mathbb{E} \|\check{x}_{k+1}-{x}^*(\bar{x})\|^2$ with
\begin{align}
&\mathbb{E}\|\check{x}_{k+1}- \left(  A+B+\sigma I\right)^{-1}(\sigma \check{x}_k) \|^2\nonumber\\ \leq &\frac{1}{4(1+\sigma/\mu)^2}\mathbb{E}\|\check{x}_k-x^*-\left(  A+B+\sigma I\right)^{-1}(\sigma \check{x}_k)+\left(  A+B+\sigma I\right)^{-1}(\sigma x^*)  \|^2\nonumber\\
=& \frac{1}{4(1+\sigma/\mu)^2}\mathbb{E}\|\check{x}_k-x^*-\left(  (A+B)\sigma^{-1}+ I\right)^{-1}( \check{x}_k)+\left(  (A+B)\sigma^{-1}+ I\right)^{-1}(x^*)  \|^2\nonumber\\
=& \frac{1}{4(1+\sigma/\mu)^2}\mathbb{E}\| \left( I- \left(  \left(A+B\right)\sigma^{-1}+ I\right)^{-1}\right)(\check{x}_k)-\left( I- \left(  \left(A+B\right)\sigma^{-1}+ I\right)^{-1}\right)(x^*)  \|^2,\label{eq: Catalyst1}
\end{align}
where the second line follows by telescoping with $x^*$. 

Now, for any maximal monotone operator $C$, $I-(I+C)^{-1}=(I+C^{-1})^{-1}.$ To prove this statement, suppose $y+C(y)=x$. Then $x-(I+C)^{-1}(x)=x-y=C(y)$. On the other hand, $C(y)+C^{-1}(C(y))=y+C(y)=x$, and so $(I+C^{-1})^{-1}(x) = C(y)$. Then $I-(I+C)^{-1}=(I+C^{-1})^{-1},$ and so \eqref{eq: Catalyst1} is equivalent to
 \begin{align*}
& \mathbb{E}\|\check{x}_{k+1}- \left(  A+B+\sigma I\right)^{-1}(\sigma \check{x}_k) \|^2\\ \leq & \frac{1}{4(1+\sigma/\mu)^2}\mathbb{E}\| \left( I+  \left(A+B\right)^{-1}\sigma \right)^{-1}(\check{x}_k)-\left( I+  \left(A+B\right)^{-1}\sigma \right)^{-1}(x^*)  \|^2.
 \end{align*}
It is immediate $\left(A+B\right)^{-1}\sigma$ is monotone since the inverse of a monotone operator is monotone.  In addition, $ \left( I+  \left(A+B\right)^{-1}\sigma\right)^{-1}$ is non-expansive since it is the resolvent of  $\left(A+B\right)^{-1}\sigma$ and the resolvent of a monotone operator is always non-expansive \citep{survey}. 
As a result, we obtain
\begin{align}\label{eq: Catalyst2}
 \mathbb{E}\|\check{x}_{k+1}- \left(  A+B+\sigma I\right)^{-1}(\sigma \check{x}_k) \|^2\leq \frac{1}{4(1+\sigma/\mu)^2} \mathbb{E} \|\check{x}_k-x^* \|^2.
\end{align}
Thus, we can use Minkowski's inequality to see
\begin{align*}
\left(\mathbb{E}\|\check{x}_{k+1}-x^*  \|^2\right)^{1/2}\leq &\left(\mathbb{E}\|\check{x}_{k+1}-  \left(  A+B+\sigma I\right)^{-1}(\sigma \check{x}_k) \| ^2 \right)^{1/2}\\& \quad +\left(\mathbb{E}\| \left(  A+B+\sigma I\right)^{-1}(\sigma \check{x}_k)-x^* \| ^2 \right)^{1/2}\\
\leq & \frac{1}{2(1+\sigma/\mu)} \left(\mathbb{E} \|\check{x}_k-x^* \|^2\right)^{1/2}+\\
&\quad +\left(\mathbb{E}\| \left(  \sigma^{-1}(A+B)+ I\right)^{-1}( \check{x}_k)-\left(  \sigma^{-1}(A+B)+ I\right)^{-1}(x^*) \| ^2 \right)^{1/2},
\end{align*}
where the first term on the RHS of the second inequality is by \eqref{eq: Catalyst2} and the second term on the RHS of the second inequality is by routine reformulation (using $ \left(  A+B+\sigma I\right)^{-1}(\sigma x)=\left(  \sigma^{-1}(A+B)+ I\right)^{-1}( x)$ for all $x$).
Since $\sigma^{-1}(A+B)$ is $\mu/\sigma$-strongly monotone by Assumption \ref{assum1}, $ \left(  \sigma^{-1}(A+B)+ I\right)^{-1}$ is $ \frac{1}{1+\sigma^{-1}\mu}$-Lipschitz and so we can rewrite the last inequality in the above display as
\begin{align*}
\left(\mathbb{E}\|\check{x}_{k+1}-x^*  \|^2\right)^{1/2}\leq & \frac{1}{2(1+\sigma/\mu)} \left(\mathbb{E} \|\check{x}_k-x^* \|^2\right)^{1/2}+\frac{\sigma}{\sigma+\mu}  \left(\mathbb{E} \|\check{x}_k-x^* \|^2\right)^{1/2}\\
&=\left( 1-\frac{1}{2(1+\sigma/\mu)}  \right)\left(\mathbb{E} \|\check{x}_k-x^* \|^2\right)^{1/2}.
\end{align*}
The desired result follows by squaring both sides.

\section{Proofs for Section \ref{sec:asyn} (asynchronous implementation)}

\subsection{Proof of Lemma \ref{asyn3} (one-step expected error)}

Use the definition of $x_{k+1}$ to see that
\begin{align*}
\|x_{k+1}-x^*\|^2 &= \| x_{k}-x^*-\gamma \mathcal{G}(\hat{x}_k,\hat{\phi}^k, I_{k})      \|^2\\
&= \|x_{k}-x^*\|^2-2\gamma\langle\mathcal{G}(\hat{x}_k,\hat{\phi}^k, I_{k}),x_{k}-x^*\rangle+\gamma^2 \|  \mathcal{G}(\hat{x}_k,\hat{\phi}^k, I_{k})\|^2\\
&=\|x_{k}-x^*\|^2-2\gamma\langle\mathcal{G}(\hat{x}_k,\hat{\phi}^k, I_{k}),\hat{x}_k-x^*\rangle+\gamma^2 \|  \mathcal{G}(\hat{x}_k,\hat{\phi}^k, I_{k})\|^2\\
&~~~~+2\gamma\langle \mathcal{G}(\hat{x}_k,\hat{\phi}^k, I_{k}),\hat{x}_k-x_{k} \rangle .
\end{align*}
Here, the second equality is from a direct expansion of the squared norm and the third equality is by adding and subtracting $\hat{x}_k$.
Due to the unbiasedness of $\mathcal{G}(\hat{x}_k,\hat{\phi}^k, I_{k})$ and the strong monotonicity of $B$, we have
\begin{equation}
\label{eq:main 2}
 \mathbb{E}\left[\langle\mathcal{G}(\hat{x}_k,\hat{\phi}^k, I_{k}),\hat{x}_k-x^*\rangle|\mathcal{F}_{k}\right]=\langle B(\hat{x}_k)-B(x^*),\hat{x}_k-x^*\rangle \geq \mu \|\hat{x}_k-x^* \|^2.
\end{equation}
Use the triangle inequality to see
\begin{align*}
\frac{\mu}{2}\|x_{k}-x^*\|^2 &=\frac{\mu}{2}\|x_{k}-\hat{x}_k+\hat{x}_k-x^* \|^2\leq \mu\left( \|x_{k}-\hat{x}_k \|^2+\|\hat{x}_k-x^*  \|^2\right),
\end{align*}
and as a result we have
\begin{equation}\label{eq:main3}
\mu\|\hat{x}_k-x^*  \|^2\geq \frac{\mu}{2}\|x_{k}-x^* \|^2-\mu\|x_{k}-\hat{x}_k \|^2.
\end{equation}
Combine (\ref{eq:main 2}) and (\ref{eq:main3}) to get
\begin{align*}
 \mathbb{E}\left[ \|x_{k+1}-x^*\|^2 |\mathcal{F}_{k} \right]\leq &(1-\gamma\mu)\|x_{k}-x^*\|^2+2\gamma\mathbb{E}\left[\langle\mathcal{G}(\hat{x}_k,\hat{\phi}^k, I_{k}),\hat{x}_k-x_{k}\rangle| \mathcal{F}_k\right]\\
 &~~+2\gamma\mu \|\hat{x}_k -x_{k}\|^2+\gamma^2\mathbb{E}\left[\|\mathcal{G}(\hat{x}_k,\hat{\phi}^k, I_{k})\|^2|\mathcal{F}_k\right],
\end{align*}
then take expectations of both sides of the above inequality to establish (\ref{asyn3.5}).

\subsection{Proof of Lemma \ref{asyn4} (error bounds)}

We have $\|\mathcal{G}(\hat{x}_k,\hat{\phi}^k, I_{k})\|^2\leq M^2 $ by the definition of $M$. Since $\hat{x}_k=x_{D(k)}$, we may use (\ref{asyn2.5}) to see that
 \begin{align}\label{eq:x_k-wide x_k}
 \|x_{k}-\hat{x}_k\|^2=\gamma^2\|\mathcal{G}(\hat{x}_{D(k)},\hat{\phi}^{D(k)},I_{D(k)})+\cdots+\mathcal{G}(\hat{x}_{k-1},\hat{\phi}^{k-1},I_{k-1}) \|^2\leq \gamma^2\tau^2M^2.
 \end{align}
The inequality above is due to the fact that the number of terms inside the norm is not greater than $\tau$ from Assumption \ref{assum3}.
To bound $\langle\mathcal{G}(\hat{x}_k,\hat{\phi}^k, I_{k}),\hat{x}_{k}-x_{k}\rangle$,
we use (\ref{eq:x_k-wide x_k}) to see that
$$
\langle\mathcal{G}(\hat{x}_k,\hat{\phi}^k, I_{k}),\hat{x}_{k}-x_{k}\rangle 
\leq  \| \mathcal{G}(\hat{x}_k,\hat{\phi}^k, I_{k})\|  \|\hat{x}_k-x_k\|\\
\leq \gamma\tau M^2.
$$
The first inequality above is from the Cauchy-Schwarz inequality, and the second is from (\ref{eq:x_k-wide x_k}) and the fact $\|\mathcal{G}(\hat{x}_k,\hat{\phi}^k, I_{k})\|\leq M  $ for all $k\geq 0$. 

We now bound $\mathbb{E}\|\mathcal{G}(\hat{x}_k,\hat{\phi}^k, I_{k})\|^2$.
Use the fact that $B(x^*)=0$ to see that
\begin{align*}
\mathbb{E}\left[\|\mathcal{G}(\hat{x}_k,\hat{\phi}^k, I_{k})\|^2|\mathcal{F}_{k}\right]&= \mathbb{E}\left[\|\mathcal{G}(\hat{x}_k,\hat{\phi}^k, I_{k})-B(\hat{x}_{k})+B(\hat{x}_{k})-B(x^*) \|^2 |\mathcal{F}_{k}\right]\\
&= \mathbb{E}\left[\|\mathcal{G}(\hat{x}_k,\hat{\phi}^k, I_{k})-B(\hat{x}_{k})\|^2 |\mathcal{F}_{k}\right]+ \|B(\hat{x}_{k})-B(x^*) \|^2 \\
&\leq  \mathbb{E}\left[\|B_{I_{k}}(\hat{x}_{k})-\hat{\phi}_{I_{k}}^{k} \|^2 |\mathcal{F}_{k} \right]+L^2\|\hat{x}_{k}-x^* \|^2,
\end{align*}
where the second equality above is due to the unbiasedness of $\mathcal{G}(\hat{x}_k,\hat{\phi}^k, I_{k})$.
By adding and subtracting $B_{I_k}(x^*)$ in the first term of the last inequality, we can further bound the above display with
\begin{align*}
\mathbb{E}\left[\|\mathcal{G}(\hat{x}_k,\hat{\phi}^k, I_{k})\|^2|\mathcal{F}_{k}\right]
&\leq  \mathbb{E}\left[\|B_{I_{k}}(\hat{x}_{k})-B_{I_{k}}(x^*)+B_{I_{k}}(x^*)-\hat{\phi}_{I_{k}}^{k} \|^2 |\mathcal{F}_{k}\right]+L^2\|\hat{x}_{k}-x^* \|^2\\
&\leq 3L^2\|\hat{x}_{k}-x^* \|^2+2 \mathbb{E}\left[\|B_{I_{k}}(x^*)-\hat{\phi}_{I_{k}}^{k} \|^2 |\mathcal{F}_{k} \right] \\
&\leq 3L^2( 2\|\hat{x}_{k}-x_{k} \|^2+2\|x_{k}-x^* \|^2)\\
&~~+2\left(2 \mathbb{E}\left[\|B_{I_{k}}(x^*)-\phi_{I_{k}}^{k} \|^2 |\mathcal{F}_{k}  \right]
+2 \mathbb{E}\left[\|\phi_{I_{k}}^{k}-\hat{\phi}_{I_{k}}^{k} \|^2 |\mathcal{F}_{k} \right]\right) \\
&=6L^2 \|\hat{x}_{k}-x_{k} \|^2+6L^2 \|x_{k}-x^* \|^2 +\frac{4}{n}\sum_{i=1}^n     \|\phi_{i}^k-B_{i}(x^*)\|^2+\\
&~~~~\frac{4}{n}\sum_{i=1}^n\|\phi_{i}^k-\hat{\phi}_{i}^{k}\|^2.
\end{align*}
Here, the second and third inequalities are due to the triangle inequality. 
The last equality is by direct calculation of the conditional expectation.
Use the fact that $\mathbb{E} \|x_{k}-\hat{x}_k\|^2 \leq \gamma^2\tau^2M^2$ and take expectations to conclude.

\subsection{Proof of Theorem \ref{thm:asyn} (asynchronous convergence rates)}

We first show that $\theta<1$. 
With the choices of $\gamma$ and $\rho$ given in the theorem, we have that $\theta$ is the maximum of $\ell_1$ and $\ell_2$ where
$$
\ell_1 \triangleq  1-\gamma\mu+6\gamma^2L^2+c_2\gamma^{1.5},~~ \ell_2 \triangleq 4\gamma^{0.5}+c_1.
$$
It suffices to show that both $\ell_1$ and $\ell_2$ are less than one.
By decreasing the denominators of both terms inside the bracket of \eqref{eq:asyn gamma ineq}, we have
\begin{equation}\label{eq:asyn gamma refined ineq}
\gamma<\min\left\lbrace \left(\frac{\mu}{\frac{3(1-c_1)L^2}{2}+c_2}\right)^2, ~ \left(\frac{1-c_1}{4} \right)^2\right\rbrace.
\end{equation}
It is readily seen that $\ell_2<1$ because $\gamma<\left(1-c_1\right)^2/4$ by the above display. 
Similar to the proof of Theorem \ref{main theorem}, we can use (\ref{eq:asyn gamma refined ineq}) to show $\mu - 6\gamma L^2 - c_2\gamma^{0.5}>0$.
By expressing $\ell_1$ as $1-\gamma(\mu - 6\gamma L^2 - c_2\gamma^{0.5})$,  we have $\ell_1<1$ and it follows that $\theta<1.$

Next, we establish inequality (\ref{eq:asyn L_rho contraction}).
Based on Assumption \ref{assum4} and Corollary \ref{asyn_coro}, we have
\begin{align*}
\mathbb{E}L_{\rho}\left(\tilde{x}_{k},\tilde{\phi}^{k}\right)
&= \mathbb{E}\left[\|x_{S_k}-x^{*}\|^{2}+\rho G\left(\phi^{S_k}\right) \right]\\
& \leq \left(1-\gamma \mu+6\gamma^{2}L^2+c_2\rho\right) \mathbb{E} \|x_{S_{k}-1}-x^{*}\|^{2} 
+\left(\frac{4\gamma^2}{\rho}+c_1\right)\rho  \mathbb{E} G\left(\phi^{S_k-1}\right) 
\\
&~~~+4\gamma^2 \mathbb{E} H\left(\phi^{S_k-1}\right)+\underbrace{4\gamma^2\mathcal{E}_2+\mathcal{E}_0+\rho\mathcal{E}_1}_{\mathcal{E}_3} \\
&\leq \max \left\lbrace 1-\gamma\mu+6\gamma^2L^2+c_2\rho,~\frac{4\gamma^2}{\rho}+c_1  \right\rbrace 
\mathbb{E} L_\rho\left(x_{S_k-1},\phi^{S_k-1}\right)\\
&~~~ + 4\gamma^2 \mathbb{E} H(\phi^{S_k-1})+\mathcal{E}_3 .
\end{align*}
Using the definition of $\theta$, the above display is equivalent to
$$
\mathbb{E}L_{\rho}\left(\tilde{x}_{k},\tilde{\phi}^{k}\right)\leq
\theta \mathbb{E}L_\rho\left(x_{S_k-1},\phi^{S_k-1}\right) +4\gamma^2 \mathbb{E} H\left(\phi^{S_k-1}\right)+\mathcal{E}_3.
$$
By recursively applying the above inequality and using the facts that $\theta<1$ and $m_k\leq\bar{m}$ for all $k \geq 0$, we have
\begin{align*}
\mathbb{E} L_\rho\left(x_{S_k},\phi^{S_k}\right) &\leq \left(\theta+4\gamma^2\bar{m}c_3\right) \mathbb{E} L_\rho\left(x_{S_{k-1}},\phi^{S_{k-1}}\right)+\sum_{i=0}^{m-1}\theta^i\mathcal{E}_3\\
&\leq \lambda \mathbb{E}L_\rho \left(x_{S_{k-1}},\phi^{S_{k-1}}\right) +\frac{\mathcal{E}_3}{1-\theta}.
\end{align*}
We finish the proof by showing that $\lambda<1.$ 
Using the definition of $\theta$, we may rewrite $\lambda$ as
\begin{align*}
 \lambda=\max \left\lbrace1-\gamma\mu+6\gamma^2L^2+c_2\rho+4\gamma^2\bar{m}c_3 , \frac{4\gamma^2}{\rho}+c_1+4\gamma^2\bar{m}c_3\right\rbrace.
\end{align*}
With $\rho=\gamma^{1.5}$, $\lambda$ is the maximum of $\tilde{\ell}_1$ and $\tilde{\ell}_2$ where 
$$
\tilde{\ell}_1 \triangleq 1-\gamma\mu+6\gamma^2L^2+c_2\gamma^{1.5} +4\gamma^2\bar{m}c_3,~~~\tilde{\ell}_2 \triangleq 4\gamma^{0.5}+4\gamma^2\bar{m}c_3+c_1.
$$
It suffices to prove that both $\tilde{\ell}_1$ and $\tilde{\ell}_2$ are less than one. 
We may rewrite $\tilde{\ell}_2$ as $c_1+4\gamma^{0.5}(1+\gamma^{1.5}\bar{m}c_3)$. 
Then, similar to the proof of Theorem \ref{main theorem}, we can use the fact that $\gamma$ is less than the second term in the brackets of (\ref{eq:asyn gamma ineq}) and (\ref{eq:asyn gamma refined ineq}) to see that $\tilde{\ell}_2<1.$
To establish $\tilde{\ell}_1<1$, note that
\begin{align*}
6\gamma L^2+c_2\gamma^{0.5}+4\gamma \bar{m}c_3&=\gamma^{0.5}\left(6\gamma^{0.5}L^2+4\bar{m}c_3\gamma^{0.5}+c_2\right)\\
&< \frac{\mu}{\frac{3(1-c_1)L^2}{2}+(1-c_1)c_3\bar{m}+c_2}\left(6\gamma^{0.5}L^2+4\bar{m}c_3\gamma^{0.5}+c_2\right),
\end{align*}
where the inequality follows because $\gamma$ is less than the first term in the brackets of (\ref{eq:asyn gamma ineq}). 
Since $\gamma< (1-c_1)^2/16$, the above display is less than $\mu$. 
Finally, rewrite $\tilde{\ell}_1$ as $1+\gamma ( 6\gamma L^2+c_2\gamma^{0.5}+4\gamma \bar{m}c_3 -\mu )$ to see that $\tilde{\ell}_1<1$.

\section{Technical details for Section \ref{sec:examples} (examples)}\label{appendix:example}

\subsection{Technical details for SVRG}\label{details of svrg}
\paragraph*{Basic complexity}
First define the constants $q\triangleq 1-2\gamma\mu+3\gamma^2L^2$ and $p\triangleq 2\gamma^2L^2$.
Recall $\tilde{x}_k=x_{km}$, so $ \phi_i^t=B_i(\tilde{x}_k)$ for all $km \leq t <(k+1)m$ and $i \in [n].$
We choose $\gamma=\mu/(3L^2)$ so that $q = 1-1/3\kappa^2$.
Then, Corollary \ref{coro1} implies
\begin{align*}
\mathbb{E}\| \tilde{x}_{k+1}-x^* \|=\mathbb{E}\| x_{(k+1)m}-x^* \|^2\leq p\mathbb{E}\|x_{(k+1)m-1} -x^* \|^2+q\mathbb{E}\|\tilde{x}_k-x^*  \|^2.
\end{align*}
Recursively applying the above inequality shows that
\begin{align*}
\mathbb{E}\|\tilde{x}_{k+1}-x^*  \|^2\leq q^m \mathbb{E}\|{x}_{km}-x^*  \|^2+p(1+q+\cdots+q^{m-1}) \mathbb{E}\| \tilde{x}_{k}-x^* \|^2\leq (q^m+\frac{p}{1-q} )\mathbb{E}\| \tilde{x}_{k}-x^* \|^2,
\end{align*}
where the second inequality follows because $q=1-1/3\kappa^2<1$.
Furthermore, by routine calculation we have
\begin{align*}
q^m+\frac{p}{1-q}=\left( 1-\frac{1}{3\kappa^2} \right)^m+\frac{2}{3}.
\end{align*}
We choose $m=\frac{\log \frac{1}{12}}{\log \left(1-\frac{1}{3\kappa^2}\right) }$ so that $q^m+\frac{p}{1-q}=\frac{3}{4} $, i.e., the expected distance to $x^*$ will contract by a factor of $3/4$ after every epoch.
In large-scale problems, $\kappa$ is large and so $m=O(\kappa^2)$ due to the fact that $\log(1+x)\approx x$ for small $x$. In this case, the total complexity is $O\left(n+\kappa^2  \right)\log(1/\epsilon)$.

\paragraph*{Catalyst complexity}
According to \eqref{prop:compl of cata}, the total complexity of SVRG with Catalyst acceleration is
\begin{align*}
O\left(\frac{\sigma+\mu}{\mu} \left(n+\left(\frac{L+\sigma}{\mu+\sigma}\right)^2  \right)\log\left( \frac{2(\mu+\sigma)}{\mu}\right)\log\left(1/\epsilon\right)\right).
\end{align*}
Letting $\lambda\triangleq \sigma/\mu$, the above term can be rewritten as
\begin{align*}
O\left(\left(\lambda+1  \right) \left(n+\left(\frac{\kappa+\lambda}{1+\lambda}\right)^2  \right)\log\left( 2(1+\lambda)\right)\log\left(\frac{1}{\epsilon}\right)\right).
\end{align*}
We omit the logarithmic terms and only consider $O\left(\left(\lambda+1  \right) \left(n+\left(\frac{\kappa+\lambda}{1+\lambda}\right)^2  \right)\right)$. We define $f(\lambda)\triangleq \left(\lambda+1  \right) \left(n+\left(\frac{\kappa+\lambda}{1+\lambda}\right)^2  \right)$ and optimize over $\lambda\geq 0$ to achieve the optimal complexity. It is straightforward to verify that $f(\lambda)$ attains its minimum at $\lambda^*=\sqrt{\frac{(\kappa-1)^2-2}{n+1}}$ where $f(\lambda^*)=2\sqrt{(n+1)\left(\kappa-1)^2-2\right)}+2\kappa-2$, which is $O(\kappa\sqrt{n})$.

\subsection{Technical details for SAGA}\label{details of saga}
\paragraph*{Basic complexity}

Recall from \eqref{eq: contract of L} that
\begin{align}\label{eq: analy of complex of saga}
\mathbb{E}L_{\rho}(x_{k},\phi^k)\leq \max \left\{ 1-2\gamma\mu+3\gamma^2L^2+\frac{\rho L^2}{n},\frac{2\gamma^2}{\rho}+1-\frac{1}{n} \right\}\mathbb{E}L_{\rho}(x_{k-1},\phi^{k-1}),\,\forall k \geq 0.
\end{align}
By choosing $\rho=4\gamma^2n$ and $\gamma=\mu/(7L^2)$, the above display becomes
\begin{align*}
\mathbb{E}L_{\rho}\left(x_{k+1},\phi^{k+1}  \right)\leq \max\left\{ 1-\frac{1}{7\kappa^2}, 1-\frac{1}{2n}\right\}L_{\rho}\left(x_{k},\phi^{k}  \right).
\end{align*} 
Then, after
$$
k=\max \left\{\frac{\log\left(\epsilon/L(x_0,\phi^0)\right)}{\log \left( 1-\frac{1}{7\kappa^2}\right)},\frac{\log\left(\epsilon/L(x_0,\phi^0)\right)}{\log \left( 1-\frac{1}{2n}\right)}    \right\}=O\left(\max\{ n,\kappa^2 \}\right)\log\left(\frac{1}{\epsilon}\right)
$$
iterations, we have $\mathbb{E}L_{\rho}\left(x_{k},\phi^{k}  \right)\leq \epsilon $.
We require two operator evaluations for each of the $k$ iterations, and so the overall complexity is $O\left(\max\{ n,\kappa^2 \}\right)\log(1/\epsilon)$.

\paragraph*{\textbf{Catalyst complexity}}
The total complexity of SAGA with Catalyst acceleration is 
\begin{align*}
O\left(\frac{\sigma+\mu}{\mu}\max\left\{n,\left(\frac{L+\sigma}{\mu+\sigma}\right)^2  \right\}\log\left( \frac{2(\mu+\sigma)}{\mu}\right)\log\left(1/\epsilon\right)\right).
\end{align*}
Let $\lambda\triangleq \sigma / \mu$ and rewrite the above display as
\begin{align*}
O\left(\left(\lambda+1 \right)\max\left\{n,\left(\frac{\kappa+\lambda}{1+\lambda}\right)^2  \right\}\log\left(2(\lambda+1)\right)\log\left(\frac{1}{\epsilon}\right)\right).
\end{align*}
We minimize $f(\lambda)\triangleq \left(\lambda+1 \right)\max\left\{n,\left(\frac{\kappa+\lambda}{1+\lambda}\right)^2  \right\} $ over $\lambda\geq 0$ to obtain the optimal complexity of SAGA.
For any $a,\, b \geq 0$ we have $(a+b)/2\leq \max\{a,b \}\leq a+b$, and so it follows that
\begin{align*}
\frac{(\lambda+1)\left(n+\left(\frac{\kappa+\lambda}{1+\lambda}\right)^2  \right)}{2}\leq f(\lambda)\leq (\lambda+1)\left(n+\left(\frac{\kappa+\lambda}{1+\lambda}\right)^2\right) .
\end{align*}
The RHS and LHS above both have optimal values with order $O(\kappa\sqrt{n})$ when $\lambda = \lambda^*=\sqrt{\frac{(\kappa-1)^2-2}{n+1}}$, and so $f(\lambda^*) = O(\kappa\sqrt{n})$.

\paragraph*{Asynchronous implementation}
We now verify Assumption \ref{assum4}.1. It can be shown that
$$ 
\mathbb{E}G(\phi^k)\leq  \left(1-\frac{1}{n}\right)\mathbb{E}G(\phi^{k-1})+\frac{L^2}{n} \mathbb{E}\|\hat{x}_{k}-x^{*}\|^{2},\,\forall k \geq 1.
$$
By adding and subtracting $x_{k-1}$ in the second term of the above inequality, we have
\begin{align*}
\mathbb{E}G(\phi^k)&\leq  \left(1-\frac{1}{n}\right)\mathbb{E}G(\phi^{k-1})+\frac{L^2}{n} \mathbb{E}\|\hat{x}_{k-1}-x_{k-1}+x_{k-1}-x^{*}\|^{2}\\
&\leq \left(1-\frac{1}{n}\right)\mathbb{E}G(\phi^{k-1})+\frac{2L^2\mathbb{E}\|\hat{x}_{k-1}-x_{k-1}\|^2}{n}+\frac{2L^2\mathbb{E}\|x_{k-1}-x_{k-1}\|^2}{n}\\
&\leq \left(1-\frac{1}{n}\right)\mathbb{E}G(\phi^{k-1})+\frac{2L^2}{n}\mathbb{E}\|x_{k-1}-x_{k-1}\|^2+\frac{2L^2\gamma^2\tau^2M^2}{n}.
\end{align*}
The second inequality above follows from the triangle inequality, and the last inequality follows from Lemma \ref{asyn4}.

\subsection{Technical details for SVRG-rand}
\paragraph*{Verifying Assumption \ref{assum2}}
We first verify that SVRG-rand satisfies Assumption \ref{assum2} and has a linear convergence rate.
\begin{theorem}\label{thm:svrgr}
Suppose Assumption \ref{assum1} holds and that $\underline{p}>0$, and set $\gamma$ and $\rho$ as in Theorem \ref{main theorem}. 
Then, the sequence $\{x_k\}_{k\geq0}$ produced by SVRG-rand satisfies \eqref{eq:x_k contraction}.
\end{theorem}
\begin{proof}
It suffices to verify Assumption \ref{assum2} and then use Theorem \ref{main theorem}.
Let $\mathcal{S}=[n]$, then Assumption \ref{assum2}.2 holds automatically with $c_3=0$ and $m_i=1$ for all $i\geq 1$. 
To verify Assumption \ref{assum2}.1, we observe that the proxy $\phi_i^k$ has a probability $p_{k-1}$ of changing to $B_i(x_{k-1})$ and a probability $1-p_{k-1}$ of remaining as $\phi_i^{k-1}$ for all $i \in [n]$.
Take conditional expectations to see that
\begin{align*}
\mathbb{E}\left[ G(\phi^k)|\mathcal{F}_{k-1}\right] &=\frac{p_{k-1}}{n}\sum_{i=1}^n \| B_{i}(x_{k-1})-B_{i}(x^*)\|^2+\frac{1-p_{k-1}}{n}\sum_{i=1}^n \|\phi_{i}^{k-1}-B_{i}(x^*)\|^2.
\end{align*}
Then use Lipschitz continuity of each $B_i$ and the definition of $\underline{p}$ to conclude that
\begin{align*}
 \mathbb{E}G(\phi^k) &\leq (1-p_{k-1}) \mathbb{E}G(\phi^{k-1})+p_{k-1}L^2  \mathbb{E}\|x_{k-1}-x^*\|^2\\
 &\leq (1-\underline{p}) \mathbb{E}G(\phi^{k-1})+\overline{p}L^2  \mathbb{E}\|x_{k-1}-x^*\|^2,
\end{align*}
for all $k \geq 0$. Therefore, SVRG-rand satisfies Assumption \ref{assum2} with $\mathcal{S} = [n]$, 
$m_i=1$ for all $i\geq1$, $c_3=0$, $c_1=1-\underline{p}$, and $c_2=\overline{p}L^2$. 
\end{proof}

\paragraph*{Basic and Catalyst complexity}
When $\overline{p}=\underline{p}=1/n$, Theorem \ref{thm:svrgr+saga} shows that \eqref{contract of G for saga} holds and so \eqref{eq: analy of complex of saga} also holds for SVRG-rand.
We then use the same arguments as for SAGA to derive the complexity of SVRG-rand before and after Catalyst acceleration.

\subsection{Technical details for HSAG}\label{details of hsag}
\paragraph*{Verifying Assumption \ref{assum2}}
We first verify Assumption \ref{assum2}.1. Note that $G(\phi^{k-1})$ changes only when $I_{k-1}$ takes a value in $\mathcal{S}$. 
Taking conditional expectation shows that for all $k\geq 1$,
\begin{align*}
\mathbb{E}\left[G(\phi^k)|\mathcal{F}_{k-1}\right] &= \frac{1}{n}\sum_{i\notin \mathcal{S}} G(\phi^{k-1})\\
&+\frac{1}{n}\sum_{i \in  \mathcal{S}} \left(G\left(\phi^{k-1}\right)+\frac{1}{n}\left(\|B_{i}(x_{k-1})-B_{i}(x^*)\|^2-\|\phi_{i}^{k-1}-B_{i}(x^*) \|^2\right)\right)\\
&= G\left(\phi^{k-1}\right)+\frac{1}{n^2}\sum_{i \in  \mathcal{S}} \left(\|B_{i}(x_{k-1})-B_{i}(x^*)\|^2-\|\phi_{i}^{k-1}-B_{i}(x^*) \|^2\right)\\
&=\left(1-\frac{1}{n}\right)G\left(\phi^{k-1}\right)+\frac{1}{n^2}\sum_{i \in \mathcal{S}}\|B_{i}(x_{k-1})-B_{i}(x^*)\|^2.
\end{align*}
Then, use the Lipschitz continuity of each $B_i$ to conclude that for all $k\geq 1$,
\begin{equation} \label{eq:saga on s}
\mathbb{E}G(\phi^k)\leq \left(1-\frac{1}{n}\right) \mathbb{E}G(\phi^{k-1})+\frac{SL^2}{n^2}\mathbb{E}\|x_{k-1}-x^{*}\|^{2}. 
\end{equation}
Next we check Assumption \ref{assum2}.2. 
For $i \notin \mathcal{S}$ and $mk\leq t<m(k+1)$, we have $\phi_{i}^t=B_i(x_{mk})$ 
and
\begin{equation*}
H(\phi^t) = \frac{1}{n}\sum_{i\notin \mathcal{S}}\| B_i(x_{mk})-B_i(x^*) \|^2 \leq  \frac{\left(n-S  \right)L^2}{n} \|x_{mk}-x^*  \|^2 \leq \frac{\left(n-S  \right)L^2}{n} L_{\rho}(x_{mk},\phi^{mk}).
\end{equation*}
The first inequality is due to Lipschitz continuity of each $B_i$,
and the second inequality is due to the definition of $L_\rho$. HSAG then satisfies Assumption \ref{assum2} with $c_1=1-1/n$, $c_2=SL^2/n$, $c_3=(n-S)L^2/n$, and $ m_j = m$ for all $j\geq 1$.

\paragraph*{Basic complexity}
Define
\begin{align}
q\triangleq \max\left\{ 1-2\gamma \mu+3\gamma^{2}L^2+\rho \frac{SL^2}{n^2}, \left(\frac{2\gamma^2}{\rho}+1-\frac{1}{n}\right)\mathbb{I}(S>0) \right\},\quad p\triangleq 2\gamma^2\frac{(n-S)L^2}{n}. \label{def q p hsag}
\end{align}
When $S=0$, HSAG becomes SVRG and the above constants $q$ and $p$ are the same as for SVRG.
Then, by \eqref{eq: exact contract of L} we have
\begin{align}
\mathbb{E}L_{\rho}(\tilde{x}_{k+1},\tilde{\phi}_{k+1}) &\leq \left(q^m +p(1+q+q^2+\cdots+q^{m-1})\right)\mathbb{E}L_{\rho}(\tilde{x}_k,\tilde{\phi}_k).\label{eq: hsag contarct}	
\end{align}
 Now we choose the step size $\gamma = \lambda\, \mu / L^2$ where
\begin{align}
\lambda=\min \left\{\frac{\kappa}{\sqrt{6n}}\mathbb{I}(0<S<n)+\mathbb{I}(S=0)+\mathbb{I}(S=n),\frac{1}{3+4S/n}  \right\}.\label{eq: choice lambda}
\end{align}
When $S=0$, the step size is the same as our choice for SVRG; when $S=n$, it is the same as our choice for SAGA.
With the above selections of $\lambda$, $\gamma=\lambda\, \mu / L^2$, and $\rho=4\gamma^2n$, we first show $q<1$ so that the summation in \eqref{eq: hsag contarct} is bounded.
First, we see that
\begin{align}\label{eq: q specific}
q= \max\left\{ 1-\frac{\lambda\left(2-\left( 3+4S/n \right)\lambda\right)}{\kappa^2}, \left(1-\frac{1}{2n}\right)\mathbb{I}(S>0) \right\}.
\end{align}
The second term inside the brackets is always less than one.
By \eqref{eq: choice lambda}, $\lambda\leq \frac{1}{3+4S/n} $ and so the first term inside the brackets above is also less than one.
Combine both cases to conclude that $q<1$.
Then, \eqref{eq: hsag contarct} implies that
\begin{align}\label{eq: contract of hsag}
\mathbb{E}L_{\rho}(\tilde{x}_{k+1},\tilde{\phi}_{k+1}) &\leq \left(q^m +\frac{p}{1-q}\right) \mathbb{E}L_{\rho}(\tilde{x}_k,\tilde{\phi}_k).
\end{align}
By the definition of $q$ above and $p$ in \eqref{def q p hsag}, we compute
\begin{align}
\frac{p}{1-q}&=\max\left\{\frac{2\lambda\left(\frac{n-S}{n}\right)}{2-(3+4\frac{S}{n})\lambda} , \left(\frac{4\lambda^2(n-S)}{\kappa^2}\right)\mathbb{I}(S>0)  \right\}\label{eq: p/1-q}.
\end{align}
With $\lambda$ given by \eqref{eq: choice lambda}, we claim that $\frac{p}{1-q}\leq \frac{2}{3}$.
When $S=n$,  we have $p / (1-q)=0$; when $S=0$, we have $\frac{p}{1-q}=\frac{2\lambda}{2-3\lambda} $.
Since $\lambda\leq \frac{1}{3}$ by \eqref{eq: choice lambda} and the mapping $\lambda\mapsto \frac{2\lambda}{2-3\lambda}$ is monotone, we have $\frac{p}{1-q}\leq \frac{2}{3} $.

Now we consider the case where $0<S<n$ and show that both terms inside the brackets of \eqref{eq: p/1-q} are less than $\frac{2}{3}$. In this case   $\lambda=\min \left\{\frac{\kappa}{\sqrt{6n}},\left(\frac{1}{3+4S/n}\right)  \right\}$ from \eqref{eq: choice lambda} and so $\lambda\leq \frac{1}{3+4S/n}$. Use the monotonicity of the mapping $\lambda\mapsto \frac{2\lambda\left(\frac{n-S}{n}\right)}{2-(3+4S/n)\lambda}$ to see that the first term inside the brackets of \eqref{eq: p/1-q} is less than $\frac{2}{3}$.
On the other hand, since $\lambda=\min \left\{\frac{\kappa}{\sqrt{6n}},\left(\frac{1}{3+4S/n}\right)  \right\}$, we also have $\lambda\leq \frac{\kappa}{\sqrt{6n}} $.
Then, the second term inside the brackets of \eqref{eq: p/1-q} can be bounded with
\begin{align*}
\frac{4\lambda^2(n-S)}{\kappa^2}\leq \frac{2(n-S)}{3n}<\frac{2}{3}.
\end{align*} 
Combining all of these cases, we have shown $\frac{p}{1-q}\leq \frac{2}{3} $ with the choice of $\lambda$ in \eqref{eq: choice lambda}, $\gamma=\lambda\, \mu / L^2$, and $\rho=4\gamma^2n$.
Recall $q$ can be written as \eqref{eq: q specific} with the choice of $\lambda$ in \eqref{eq: choice lambda}, $\gamma = \lambda\mu /L^2$, and $\rho=4\gamma^2n$.
By setting 
\begin{align}\label{def: m}
m=\max\left\{\frac{\log \frac{1}{12}}{\log \left( 1-\frac{2\lambda-(3+4S/n)\lambda^2}{\kappa^2} \right)} ,\frac{\log \frac{1}{12}}{\log \left( 1-\frac{1}{2n} \right)} \mathbb{I}(S>0)\right\},
\end{align}
 we have $q^m\leq \frac{1}{12} $ and \eqref{eq: contract of hsag} becomes
\begin{align*}
\mathbb{E}L_{\rho}(\tilde{x}_{k+1},\tilde{\phi}_{k+1}) &\leq\frac{3}{4} \mathbb{E}L_{\rho}(\tilde{x}_k,\tilde{\phi}_k),
\end{align*}
since $\frac{p}{1-q}\leq \frac{2}{3}$.
In other words, we have shown the expected value of the Lyapunov function contracts by $\frac{3}{4}$ after each epoch. 

We require $ O(\log(1/\epsilon))$ epochs to obtain an $\epsilon$-solution.
In each epoch, the full updates require $n-S$ operator evaluations due to the SVRG updates.
In each inner iteration, we need two operator evaluations.
Then, the total number of operator evaluations required to achieve an $\epsilon$-solution is $O((m+n-S)\log(1/\epsilon))$. 
Next, we elaborate on the order of $m$. Using the fact $\log(1+x)\approx 1+x$ for small $x$, we have
\begin{align}\label{eq: m another}
m=O\left(\max\left\{ \frac{\kappa^2}{2\lambda-(3+4S/n)\lambda^2} ,n\mathbb{I}(S>0) \right\} \right)
\end{align}
 from \eqref{def: m}. 
Recall $\lambda$ is chosen according to \eqref{eq: choice lambda}. 
If $S=0$ or $S=n$, then $\lambda=\frac{1}{3+4S/n}$, and $2\lambda-(3+4S/n)\lambda^2=\frac{1}{3+4S/n}\in [\frac{1 }{3},\frac{1}{7}] $ and so $m=O\left(\max\{\kappa^2,n\mathbb{I}(S>0)\}\right)$.
On the other hand, suppose $0<S<n$ and the minimum in \eqref{eq: choice lambda} is attained at $\frac{\kappa}{\sqrt{6n}}$.
Then $\lambda=\frac{\kappa}{\sqrt{6n}}$ and  $\frac{\kappa}{\sqrt{6n}}\leq \frac{1}{3+4\frac{S}{n}}\leq \frac{1}{3}$.
Furthermore, $ 2\lambda-(3+4S/n)\lambda^2=\lambda+\left( \lambda-(3+4S/n)\lambda^2 \right)\geq \lambda$, where the inequality uses $\lambda\leq \frac{1}{3+4S/n}.$ 
 This implies $\frac{\kappa^2}{2\lambda-(3+4S/n)\lambda^2}\leq \frac{\kappa^2}{\lambda}=\kappa\sqrt{6n}\leq 2n$ due to the condition $\frac{\kappa}{\sqrt{6n}}\leq \frac{1}{3} $. 
 Thus  $m=O(n)$ from \eqref{eq: m another}.
When $0<S<n $ and  the minimum in \eqref{eq: choice lambda} is attained at $\frac{1}{3+4S/n}$, we argue in the same was as for $S=0$ or $S=n$ to see $m=O\left(\max\{\kappa^2,n\}\right).$
We then have $m=O\left(\max\{\kappa^2,n\mathbb{I}(S>0)\} \right)$ and obtain a total complexity of
$$
O\left( \max\{\kappa^2,n\mathbb{I}(S>0)\}+n-S\right)\log(1/\epsilon).
$$

\paragraph*{Catalyst complexity}
It is readily seen that for all $S\geq 0$ we have
\begin{align*}
\max\{n,\kappa^2 \}\leq \left( \max\{\kappa^2,n\mathbb{I}(S>0)\}+n-S\right)\leq 2\max\{n,\kappa^2 \}.
\end{align*}
The LHS and RHS are both the same order of complexity as SAGA.
We then use the same arguments as for SAGA to derive the complexity of HSAG with Catalyst.

\paragraph*{\textbf{Asynchronous implementation}}
Similar to (\ref{eq:saga on s}), it can be shown that
\begin{equation*}
\mathbb{E}G(\phi^k)\leq \left(1-\frac{1}{n}\right) \mathbb{E}G(\phi^{k-1})+\frac{SL^2}{n^2}\mathbb{E}\|\hat{x}_{k-1}-x^{*}\|^{2},\, \forall k\geq 1.
\end{equation*}
Furthermore,
$$
\mathbb{E}G(\phi^k)\leq \left(1-\frac{1}{n}\right) \mathbb{E}G(\phi^{k-1})+\frac{2SL^2}{n^2}\mathbb{E}\|x_{k-1}-x^{*}\|^{2}+\frac{2SL^2\gamma^2\tau^2M^2}{n^2},\, \forall k\geq 1.
$$
Assumption \ref{assum4}.2 holds with $c_3=\left(n-S  \right)L^2/n$.
In addition, we have $\sum_{i=1}^n \mathbb{E}\|\phi_{i}^k-\hat{\phi}_{i}^{k}\|^2/n \leq 4\tau M_{\phi,B}^2/n$.

\subsection{Technical details for SAGA+SVRG-rand}\label{deyails of saga+svrg-rand}
\paragraph*{Verifying Assumption \ref{assum2}}
\begin{theorem}\label{thm:svrgr+saga}
Suppose Assumption \ref{assum1} holds and $\underline{p}>0$. Choose $\gamma$ and $\rho$ as in Theorem \ref{main theorem}, then SAGA+SVRG-rand has a linear convergence rate. 
\end{theorem}
\begin{proof}
Define functions $G_1$ and $G_2$ via
\begin{equation*}
G_1(\phi^k)\triangleq\frac{1}{n}\sum_{i\in \mathcal{S}_1}\|\phi_{i}^k-B_{i}(x^*) \|^2,\, G_2(\phi^k)\triangleq\frac{1}{n}\sum_{i\in \mathcal{S}_2}\|\phi_{i}^k-B_{i}(x^*) \|^2,
\end{equation*}
and let $S_1 = |\mathcal{S}_1|$ and $S_2 = |\mathcal{S}_2|$. It suffices to verify Assumption \ref{assum2}.
We have $\mathcal{S} = [n]$, and so $H(\phi^k)\equiv 0$ and $G(\phi^k)=G_1(\phi^k)+G_2(\phi^k)$.
Assumption \ref{assum2}.2 holds automatically with $c_3=0$ and $m_i=1$ for all $i\geq1$.
We then check Assumption \ref{assum2}.1.
From (\ref{eq:saga on s}), we have
\begin{equation*}
\mathbb{E}G_1(\phi^k)\leq \left(1-\frac{1}{n}\right) \mathbb{E}G_1(\phi^{k-1})+\frac{S_1 L^2}{n^2}\mathbb{E}\|x_{k-1}-x^{*}\|^{2},\, \forall k\geq 1.
\end{equation*}
Similar to the proof for Theorem \ref{thm:svrgr}, it can be shown that
\begin{equation*}
\mathbb{E}G_2(\phi^k)\leq \left(1-\underline{p}\right) \mathbb{E}G_2(\phi^{k-1})+\frac{\overline{p} S_2 L^2}{n}\mathbb{E}\|x_{k-1}-x^{*}\|^{2},\, \forall k\geq 1.
\end{equation*}
Combine the above two inequalities to see that
\begin{align*}
\mathbb{E}G(\phi^{k+1})&=\mathbb{E}G_1(\phi^{k+1})+\mathbb{E}G_2(\phi^{k+1})\\
&\leq  \left(1-\frac{1}{n}\right) \mathbb{E}G_1(\phi^k)+(1-\underline{p})\mathbb{E}G_2(\phi^k) +\left( \frac{S_1 L^2}{n^2}+\frac{\overline{p}S_2 L^2}{n} \right) \mathbb{E}\|x_{k}-x^{*}\|^{2}\\
&\leq \max\left\lbrace 1-\frac{1}{n},1-\underline{p}  \right\rbrace \mathbb{E}\left(G_1(\phi^{k})+G_2(\phi^{k})\right)\\
&~~~+\left( \frac{S_1 L^2}{n^2}+\frac{\overline{p}S_2 L^2}{n} \right) \mathbb{E}\|x_{k}-x^{*}\|^{2}\\
&=\max\left\lbrace 1-\frac{1}{n},1-\underline{p}  \right\rbrace  \mathbb{E}G(\phi^k)+\left( \frac{S_1 L^2}{n^2}+\frac{\overline{p}S_2 L^2}{n} \right)  \mathbb{E}\|x_{k}-x^{*}\|^{2}.
\end{align*}
%
We then conclude that the hybrid algorithm satisfies Assumption \ref{assum2} with $\mathcal{S} = [n]$, $m_i=1$ for all $i\geq1$, 
$c_3=0$, $c_1=\max\left\lbrace 1-1/n,1-\underline{p}  \right\rbrace$, and $c_2= S_1 L^2/n^2+\overline{p}S_2 L^2/n$.
\end{proof}

\paragraph*{Basic and Catalyst complexity}
When $\overline{p}=\underline{p}=1/n$, from Theorem \ref{thm:svrgr+saga} we have that \eqref{contract of G for saga} holds and so \eqref{eq: analy of complex of saga} also holds here.
We then use the same arguments as for SAGA to derive the complexity of SAGA+SVRG-rand with Catalyst.

\subsection{Technical details for SARAH}

Similar to \cite{SARAH}, we obtain linear convergence of the sequence $\|v_k  \|^2$ to zero in expectation with rate  $1-1/\kappa^2$ within each epoch.
However, \cite{SARAH} relies on the co-coercivity of $B_i$ and is restricted to convex function minimization, while we do not require co-coercivity.
To simplify, for any non-negative integer  $s\geq 0$, consider the sequence $\{y_t^s \}_{t=1}^m$ and $\{I_t ^s\}_{t=1}^m$ where
\begin{align}\label{eq: y t s}
y_{t}^s= x_{sm+t}, \quad I_t^s = I_{sm+t}, \quad 0\leq t\leq m,
\end{align}
and
$\{v_t^s \}_{t=0}^m$ is defined by
\begin{align}\label{eq: v t s}
v_{t}^s=\begin{cases}v_{sm+t} & 0\leq t\leq m-1 \\
B_{I_{t}^s}(y_t^s)-B_{I_{t}^s}(y_{t-1}^s)+v_{t-1}^s & t=m
\end{cases}.
\end{align}
In the above definitions, $s$ denotes the index of the current epoch.
We first bound ${v}_t^s$ in the following lemma.

\begin{lemma}\label{lem:linear of v}
Suppose Assumption \ref{assum1} holds and let $\{{v}_t^s\}_{t =1}^m$ be defined by \eqref{eq: v t s} for some $s\geq 0$.
Then for all $1\leq t \leq m$,
\begin{align*}
\mathbb{E} \|v_t^s  \|^2\leq (1-2\gamma\mu+\gamma^2L^2)\mathbb{E}\|v_{t-1}^s \|^2.
\end{align*}
\end{lemma}
\begin{proof}
By definition of $v_k$, $y_t^s$, and $I_t^s$ given in \eqref{eq: y t s}, we have
\begin{align*}
\|v_t^s \|^2&=\|B_{I_t^s}(y_t^s)-B_{I_t^s}(y_{t-1}^s)+v_{t-1}^s  \|^2\\
&=\| v_{t-1}^s \|^2+\|B_{I_t^s}(y_t^s)-B_{I_t^s}(y_{t-1}^s)\|^2-2\langle B_{I_t^s}(y_t^s)-B_{I_t^s}(y_{t-1}^s) , v_{k-1}^s\rangle\\
&=\| v_{t-1}^s \|^2+\|B_{I_t^s}(y_t^s)-B_{I_t^s}(y_{t-1}^s)\|^2-\frac{2}{\gamma}\langle B_{I_t^s}(y_t^s)-B_{I_t^s}(y_{k-1}^s) , y_t^s-y_{t-1}^s\rangle,
\end{align*}
using $y_t^s=y_{t-1}^s-\gamma v_{t-1}^s$.
Take conditional expectation over $\mathcal{F}_t$ (here,  $\mathcal{F}_t$ contains information up to $t$th iteration in the $s$th epoch) to see
\begin{align*}
\mathbb{E}\left[\|v_t^s \|^2\Big | \mathcal{F}_t\right]&=\| v_{t-1}^s \|^2+\frac{1}{n}\sum_{i=1}^n \|B_i(y_t^s)-B_i(y_{t-1}^s) \|^2-\frac{2}{\gamma}\langle B(y_t^s)-B(y_{t-1}^s),y_t^s-y_{t-1}^s  \rangle\\
&\leq \| v_{t-1}^s \|^2+(L^2-\frac{2\mu}{\gamma} )\|y_t^s-y_{t-1}^s \|^2\\
&=(1-2\gamma\mu+\gamma^2L^2) \|v_{t-1}^s \|^2,
\end{align*}
where the inequality is by Lipschitz continuity and strong monotonicity of $B$, and the last line is by $y_t^s=y_{t-1}^s-\gamma v_{t-1}^s$. Take full expectations to conclude.
\end{proof}

Now we bound the expectation of $\| v_t^s-B(y_t^s) \|^2$ in the following lemma.
\begin{lemma}	\label{lem sarah lem2}
Suppose Assumption \ref{assum1} holds and let $\{{v}_t^s\}_{t =1}^m$ be defined by \eqref{eq: v t s} for some $s\geq 0$. Then for all $1\leq t\leq m$,
\begin{align*}
\mathbb{E}\|v_t^s-B(y_t^s)  \|^2\leq \gamma^2 L^2\sum_{j=1}^t \mathbb{E}\| v_{j-1}^s \|^2.
\end{align*}
\end{lemma}
\begin{proof}
By telescoping with $B(y_{t-1}^s)$ and $v_{t-1}^s$, $\|v_t^s-B(y_t^s)  \|^2   $ can be written as
\begin{align*}
\|v_t^s-B(y_t^s)  \|^2&=\| \left[ B(y_{t-1}^s)-v_{t-1}^s \right]-(v_t^s-v_{t-1}^s)+\left[B(y_t^s)-B(y_{t-1}^s)\right]    \|\\
&=\|  B(y_{t-1})-v_{t-1}^s\|^2+\| v_t^s-v_{t-1}^s \|^2+\|B(y_t^s)-B(y_{t-1}^s)  \|^2\\
&\quad -2\langle  B(y_{t-1}^s)-v_{t-1}^s,  v_t^s-v_{t-1}^s\rangle+2\langle B(y_{t-1}^s)-v_{t-1}^s ,  B(y_t^s)-B(y_{t-1}^s)\rangle\\
&\quad -2\langle v_t^s-v_{t-1}^s, B(y_t^s)-B(y_{t-1}^s)  \rangle.
\end{align*}
We can compute the conditional expectation of both sides and use the fact that $\mathbb{E}[v_t^s -v_{t-1}^s | \mathcal{F}_t]=B(y_t^s)-B(y_{t-1}^s)$ to see
\begin{align*}
\mathbb{E}\left[\|v_t^s-B(y_t^s)  \|^2  \Big| \mathcal{F}_t \right]&= \|  B(y_{t-1}^s)-v_{t-1}^s\|^2+\| v_t^s-v_{t-1}^s \|^2+\|B(y_t^s)-B(y_{t-1}^s)  \|^2\\
&\quad -2\langle B(y_{t-1}^s)-v_{t-1}^s ,  B(y_t^s)-B(y_{t-1}^s)\rangle+2\langle B(y_{t-1}^s)-v_{t-1}^s ,  B(y_t^s)-B(y_{t-1}^s)\rangle\\
&\quad -2\|B(y_t^s)-B(y_{t-1}^s)  \|^2\\
&=\|  B(y_{t-1}^s)-v_{t-1}^s\|^2+\| v_t^s-v_{t-1}^s \|^2-\|B(y_t^s)-B(y_{t-1}^s)  \|^2.
\end{align*}
By taking full expectations we have
\begin{align}\label{eq:sarah_itera}
\mathbb{E}\|v_t^s-B(y_t^s)  \|^2-\mathbb{E}\|v_{t-1}^s-B(y_{t-1}^s)  \|^2=\mathbb{E}\| v_t^s-v_{t-1}^s \|^2-\|\mathbb{E}B(y_t^s)-B(y_{t-1}^s)  \|^2,
\end{align}
then summing over $j=0,1,\cdots,s$, we have
\begin{align*}
\mathbb{E}\|v_t^s-B(y_t^s)  \|^2&=\mathbb{E}\|v_0^s-B(y_0^s)  \|^2+\sum_{j=1}^t \left(\mathbb{E}\|v_j^s-B(y_j^s)  \|^2-\mathbb{E}\|v_{j-1}^s-B(y_{j-1}^s)  \|^2\right)\\
&=\sum_{j=1}^t \left(\mathbb{E}\| v_j^s-v_{j-1}^s \|^2-\mathbb{E}\|B(y_j^s)-B(y_{j-1}^s)  \|^2\right),
\end{align*}
where we used \eqref{eq:sarah_itera} and the fact that $v_0^s=B(y_0^s)$.
By definition, $v_j^s-v_{j-1}^s=B_{I_j^s}(y_j^s)-B_{I_j^s}(y_{j-1}^s)$ and so the above display implies
\begin{align*}
\mathbb{E}\|v_t^s-B(y_t^s)  \|^2\leq \sum_{j=1}^t\mathbb{E }\| B_{I_j^s}(y_j^s)-B_{I_j^s}(y_{j-1}^s) \|^2\leq L^2 \mathbb{E}\| y_j^s-y_{j-1}^s \|^2=\gamma^2L^2\mathbb{E}\| v_{j-1} ^s\|^2.
\end{align*}
The first inequality uses the non-negativity of $\mathbb{E}\|B(y_j^s)-B(y_{j-1}^s)  \|^2$, the second inequality is by Lipschitz continuity of $B_{i}$, and the last equality is by definition of $y_j^s=y_{j-1}^s-\gamma v_{j-1}^s$.
\end{proof}

For the next result, define the expression $q(\gamma) \triangleq 1-2\gamma\mu+\gamma^2L^2 $.
By the above lemma and recursive application of Lemma \ref{lem:linear of v}, we arrive at the following result.
\begin{lemma}\label{lem: sarah linear}
Suppose Assumption \ref{assum1} holds, choose $\gamma=\mu/(2L^2) $ and $m=\log_{1-\frac{3}{4\kappa^2}}(\frac{1}{24} ) $, and let $\{{y}_t^s\}_{t =0}^m$ be defined by \eqref{eq: y t s} for some $s\geq 0$. Then, we have
\begin{align*}
\mathbb{E}\| B(y_m^s) \|^2\leq \frac{3}{4}\mathbb{E}\| B(y_0^s) \|^2.
\end{align*}
\end{lemma}

\begin{proof}
When $\gamma = \mu /(2L^2)$, we have $q=1 - 3 / (4\kappa^2)<1$.
Then Lemma \ref{lem sarah lem2} shows that
\begin{align*}
\mathbb{E}\|v_t^s-B(y_t^s)  \|^2\leq \left(\gamma^2 L^2\sum_{j=1}^t q^{j-1}\right)\mathbb{E}\|v_0^s \|^2\leq \frac{\gamma^2L^2}{1-q} \mathbb{E}\| B(y_0^s) \|^2,\, \forall s\geq 1.
\end{align*}
The first inequality is by recursively applying Lemma \ref{lem:linear of v}, and the second inequality uses $v_0^s=B(y_0^s)$ and $q<1$.
In view of the above display, the Cauchy-Schwartz inequality gives
\begin{align*}
\mathbb{E}\| B(y_m^s) \|^2&\leq 2\mathbb{E}\| v_m^s \|^2+2\mathbb{E}\| v_m^s-B(y_m^s) \|^2\\
&\leq 2 q^m \mathbb{E}\| v_0^s \|^2+\frac{2\gamma^2L^2}{1-q} \mathbb{E}\| B(y_0^s) \|^2\\
&=(2q^m+\frac{2\gamma^2L^2}{1-q})\mathbb{E}\| B(y_0^s) \|^2,
\end{align*}
where the first term on the RHS of the second line is by Lemma \ref{lem:linear of v}.
Recall $q = q(\lambda) = 1-2\gamma\mu+\gamma^2L^2$, so when $\gamma=\mu/(2L^2)$ we have
\begin{align*}
2q^m+\frac{2\gamma^2L^2}{1-q}=2\left( 1-\frac{3}{4\kappa^2} \right)^m+\frac{2}{3},
\end{align*}
by routine calculations.
If $m=\log_{1-\frac{3}{4\kappa^2}}(\frac{1}{24} )$, the RHS is $3/4$ and we have
\begin{align*}
\mathbb{E}\| B(y_m^s) \|^2\leq \frac{3}{4}\mathbb{E}\| B(y_0^s) \|^2.
\end{align*}

\end{proof}
Recall we used $\tilde{x}_k$ to denote $x_{km}$, the above lemma is equivalent to 
\begin{align*}
\mathbb{E}\| B(\tilde{x}_{k+1}) \|^2\leq \frac{3}{4} \mathbb{E}\| B(\tilde{x}_{k}) \|^2
\end{align*}
for any $k\geq 0$. 
Thus we require $O(\mbox{log}(1/\epsilon)) $ epochs to obtain an $\tilde{x}_{\epsilon}^*$ such that $\mathbb{E}\| B(\tilde{x}_{\epsilon}^*) \|^2\leq \epsilon.$
The complexity within each epoch is $n+2m=O(n+\kappa^2)$, and so the overall complexity is $O(n+\kappa^2)\mbox{log}(1/\epsilon)$.

\section{Online appendices}
All the source codes can be found in our  online appendices: https://github.com/xunzhang\\1229/Variance-Reduction-algorithms-monotone-inclusion.git

\end{document}